\documentclass[letterpaper]{article} % DO NOT CHANGE THIS
\usepackage{aaai25}  % DO NOT CHANGE THIS
\usepackage{times}  % DO NOT CHANGE THIS
\usepackage{helvet}  % DO NOT CHANGE THIS
\usepackage{courier}  % DO NOT CHANGE THIS
\usepackage[hyphens]{url}  % DO NOT CHANGE THIS
\usepackage{graphicx} % DO NOT CHANGE THIS
\usepackage{natbib}  % DO NOT CHANGE THIS AND DO NOT ADD ANY OPTIONS TO IT
\usepackage{caption} % DO NOT CHANGE THIS AND DO NOT ADD ANY OPTIONS TO IT
\usepackage{algorithm}
\usepackage{algorithmic}

\usepackage{color}
\usepackage{diagbox,multirow}
\usepackage{enumitem,amsthm,amsmath,amsfonts,amssymb}

\newtheorem{theorem}{Theorem}
\newtheorem{lemma}[theorem]{Lemma}

\theoremstyle{definition}
\newtheorem{definition}{Definition}

\newtheorem{assumption}{Assumption}
\newtheorem{example}{Example}
\theoremstyle{definition}

\newtheorem{remark}{Remark}

\usepackage{newfloat}
\usepackage{listings}
\DeclareCaptionStyle{ruled}{labelfont=normalfont,labelsep=colon,strut=off} % DO NOT CHANGE THIS
\floatstyle{ruled}
\newfloat{listing}{tb}{lst}{}
\floatname{listing}{Listing}
\title{Stability-based Generalization Analysis of Randomized Coordinate Descent for Pairwise Learning \thanks{To appear in AAAI 2025. LW acknowledges support by the National Natural Science Foundation of China (72431008, 61903309) and the Sichuan Science and Technology Program (2023NSFSC1355). YL acknowledges support by the Research Grants Council of Hong Kong [Project No. 22303723].
                    }}
\author{
    Liang Wu\textsuperscript{\rm 1, 2},
    Ruixi Hu\textsuperscript{\rm 1}\thanks{The
    corresponding author is Ruixi Hu.},
    Yunwen Lei\textsuperscript{\rm 3}\\
}
\affiliations{
    \textsuperscript{\rm 1}Center of Statistical Research, School of Statistics, Southwestern University of Finance and Economics, Chengdu, China\\
    \textsuperscript{\rm 2}Big Data Laboratory on Financial Security and Behavior, SWUFE (Laboratory of Philosophy and Social Sciences, Ministry of Education), Chengdu, China\\
    \textsuperscript{\rm 3}Department of Mathematics, University of Hong Kong, Pok Fu Lam, Hong Kong\\
    wuliang@swufe.edu.cn,
    hu\_ruixi@163.com,
    leiyw@hku.hk
}

\usepackage{bibentry}
\begin{document}

\maketitle
\thispagestyle{plain}

\begin{abstract}
Pairwise learning includes various machine learning tasks, with ranking and metric learning serving as the primary representatives. While randomized coordinate descent (RCD) is popular in various learning problems, there is much less theoretical analysis on the generalization behavior of models trained by RCD, especially under the pairwise learning framework. In this paper, we consider the generalization of RCD for pairwise learning. We measure the on-average argument stability for both convex and strongly convex objective functions, based on which we develop generalization bounds in expectation. The early-stopping strategy is adopted to quantify the balance between estimation and optimization. Our analysis further incorporates the low-noise setting into the excess risk bound to achieve the optimistic bound as $O(1/n)$, where $n$ is the sample size.
\end{abstract}

% Uncomment the following to link to your code, datasets, an extended version or similar.
%
% \begin{links}
%     \link{Code}{https://aaai.org/example/code}
%     \link{Datasets}{https://aaai.org/example/datasets}
%     \link{Extended version}{https://aaai.org/example/extended-version}
% \end{links}

\section{Introduction}\label{se1}

\begin{table*}[!htbp]
\captionsetup{labelformat=empty}
\captionsetup{aboveskip=7pt}
\centering
\begin{tabular}{|c|c|c|c|c|c|c|}
\hline
Learning Paradigm & Algorithm & Reference &  Assumption & Noise Seeting & Iteration &  Rate \\
\hline
\multirow{4}{*}{Pointwise Learning} & \multirow{9}{*}{SGD} & \multirow{4}{*}{Lei (2020)} & \multirow{2}{*}{C, $L$-S} & \diagbox{}{} & \multirow{6}{*}{$T \asymp n$} & $O(1/\sqrt{n})$ \\
\cline{5-5}
\cline{7-7}
 &  &  &  & $F(\mathbf{w}^*) = 0$ &  & $O(1/n)$ \\
\cline{4-4}
\cline{5-5}
\cline{7-7}
 &  &  & C, $G$-Lip, $L$-S & \multirow{3}{*}{\diagbox{}{}} &  & $O(1/\sqrt{n})$ \\
\cline{4-4}
\cline{7-7}
 &  &  & SC, $G$-Lip, $L$-S &  &  & $O(1/n\sigma)$ \\
\cline{1-1}
\cline{3-3}
\cline{4-4}
\cline{7-7}
\multirow{8}{*}{Pairwise Learning} &  & \multirow{5}{*}{Lei (2021)} & \multirow{2}{*}{C, $L$-S} &  &  & $O(1/\sqrt{n})$ \\
\cline{5-5}
\cline{7-7}
 &  &  &  & $F(\mathbf{w}^*) = O(1/n)$ &  & $O(1/n)$ \\
\cline{4-4}
\cline{5-5}
\cline{6-6}
\cline{7-7}
 &  &  & C, $G$-Lip & \multirow{4}{*}{\diagbox{}{}} & $T \asymp n^2$ & $O(1/\sqrt{n})$ \\
\cline{4-4}
\cline{6-6}
\cline{7-7}
 &  &  & SC, $L$-S &  & $T \asymp n$ & $O(1/n\sigma)$ \\
\cline{4-4}
\cline{6-6}
\cline{7-7}
 &  &  & SC, $G$-Lip &  & $T \asymp n^2$ & $O(1/n\sigma)$ \\
\cline{2-2}
\cline{3-3}
\cline{4-4}
\cline{6-6}
\cline{7-7}
 & \multirow{3}{*}{RCD} & \multirow{3}{*}{This Work} & \multirow{2}{*}{C, $L$-S, Lip-grad} &  & $T \asymp \sqrt{n}$ & $O(1/\sqrt{n})$ \\
\cline{5-5}
\cline{6-6}
\cline{7-7}
 &  &  &  & $F(\mathbf{w}^*) = O(1/n)$ & $T \asymp n$ & $O(1/n)$ \\
\cline{4-4}
\cline{5-5}
\cline{6-6}
\cline{7-7}
 &  &  & SC, $L$-S, Lip-grad & \diagbox{}{} & $T \asymp \log (n)$ & $O(\sqrt{\log(n)}/n)$ \\
\hline
\end{tabular}
\caption{Table 1: All the above convergence rates are based on excess risk bounds in expectation. C meas the convexity as Assumption \ref{as4} and SC means the strong convexity as Assumption \ref{as5}. $G$-Lip refers to Assumption \ref{as1}, $L$-S refers to Assumption \ref{as2} and Lip-grad refers to Assumption \ref{as3}. Furthermore, Lei (2020) refers to \citeauthor{l:2020y} \shortcite{l:2020y}, Lei (2021) refers to \citeauthor{l:2021} \shortcite{l:2021}.} \label{tab1}
\end{table*}

The paradigm of pairwise learning has found wide applications in machine learning.
Several popular examples are shown as the following.
In ranking, we aim to find a model that can predict the ordering of instances \cite{c:2008x,r:2012}.
In metric learning, we wish to build a model to measure the distance between instances \cite{c:2016,y:2019,d:2020}.
Besides, various problems such as AUC maximization \cite{c:2003,g:2013,y:2016x,l:2018x} and learning tasks with minimum error entropy \cite{h:2015} can also be formulated as this paradigm.
For all these pairwise learning tasks, the performance of models needs to be measured on pairs of instances.
In contrast to pointwise learning, this paradigm is characterized by pairwise loss functions $f : \mathcal{H} \times \mathcal{Z} \times \mathcal{Z} \mapsto \mathbb{R}$, where $\mathcal{H}$ and $\mathcal{Z}$ denote the hypothesis space and the sample space respectively.
To understand and apply the paradigm better, there is a growing interest in the study under the uniform framework of pairwise learning.

Randomized coordinate descent (RCD) is one of the most commonly used first-order methods in optimization.
In each iteration, RCD updates a randomly chosen coordinate along the negative direction of the gradient and keeps other coordinates unchanged.
This makes RCD especially effective for large-scale problems \cite{n:2012}, where the computational cost is rather hard to handle.

The extensive applications of RCD have motivated some interesting theoretical analysis on its empirical behavior \cite{n:2012,r:2014,b:2021,c:2023}, which focuses on iteration complexities and empirical risks in the optimization process.
However, there is much less work considering the generalization performance of RCD, i.e., how models trained by RCD would behave on testing samples.
It is notable that the relative analysis only considers the case of pointwise learning \cite{p:2021}, which is different from pairwise learning in the structure of loss functions.
Besides, this work fails to establish the generalization bound based on the $\ell_2$ on-average argument stability in a strongly convex case.
Therefore, the existing theoretical analysis of RCD is not enough to describe the
discrepancy between training and testing for pairwise learning.
How to quantify the balance between statistics and optimization under this setting still remains a challenge.
In this paper, we develop a more systematical and fine-grained generalization analysis of RCD for pairwise learning to refine the above study.
Our analysis can lead to a more appropriate design for the optimization algorithm and the machine learning model.

In this paper, we present the generalization analysis based on the concept of algorithmic stability \cite{b:2002}.
The comparison between the existing work and this paper is presented in Table \ref{tab1}.
Our contributions are summarized as follows.

1.  Under general assumptions on $L$-smoothness of loss functions, coordinate-wise smoothness and convexity of objective functions, we study the $\ell_2$ on-average argument stability and the corresponding generalization bounds of RCD for pairwise learning.
To achieve optimal performance, we consider the balance between the generalization error and the optimization error.
The result shows that the early stopping strategy is beneficial to the generalization.
The excess risk bounds enjoy the order of $O(1 / \sqrt{n})$ and $O( \sqrt{\log(n)} / n)$ for convex and strongly-convex objective functions respectively, where $n$ denotes the sample size and $\sigma$ is the strong-convexity parameter.

2. We use the low noise condition $F(\mathbf{w}^*) = O(1/n)$ to develop shaper generalization bounds under the convex case.
This motivates the excess risk bound $O(1 / n)$, which matches the approximate optimal rate under the strongly convex case.
However, we should note that the approximate optimal rate is accessible with a faster computing $T \asymp \log(n)$ for strongly convex empirical risks.

The main work is organized according to the convexity of the empirical risk.
We consider the on-average argument stability and develop the corresponding excess risk bounds.
The early-stopping strategy is useful for balancing optimization and estimation, by which we present the optimal convergence rate.
Furthermore, there are two key points about our proof in comparison with the pointwise case \cite{p:2021}:
One is applying the coercivity property to bound the expansiveness of RCD updates since the expectation of randomized coordinates leads to the gradient descent operator.
The other is following the optimization error bounds for pointwise learning directly since they both use unbiased gradient estimations.

\section{Related Work}\label{se2}

In this section, we review the related work on RCD and generalization analysis for pairwise learning.

\textbf{Randomized Coordinate Descent (RCD).}\label{se21}
The real-world performance of RCD has demonstrated its significant efficiency in many large-scale optimization tasks, including regularized risk minimization \cite{c:2008,s:2009}, low-rank matrix completion and learning \cite{h:2019,c:2024},
and optimal transport problems \cite{x:2024}.
The convergence analysis of RCD and its accelerated variant was first proposed by \citeauthor{n:2012} \shortcite{n:2012}, where global estimates of the convergence rate were considered.
Then the strategies to accelerate RCD were further explored \cite{r:2014}, for which the corresponding convergence properties were established for structural optimization problems \cite{z:2014,l:2015,l:2017}.
RCD was also studied under various settings including nonconvex optimization \cite{b:2021,c:2023}, volume sampling \cite{r:2020} and differential privacy \cite{s:2021}.
The above study mainly considered the empirical behavior of RCD.
However, the aim of this paper is to quantify the generalization performance of machine learning models trained by RCD.

\textbf{Generalization for Pairwise Learning.}\label{se22}
The generalization ability shows how models based on training datasets will adapt to testing datasets.
It serves as an important indicator for the enhancement of models and algorithms in the view of statistical learning theory (SLT).
To investigate the generalization performance for pairwise learning, methods of uniform convergence analysis and stability analysis have been applied under this wide learning framework.
More details are described below.

The uniform convergence approach considers the connection between generalization errors and U-statistics, from which generalization bounds via corresponding U-processes are developed sufficiently.
Complexity measures including VC dimension \cite{v:1994}, covering numbers \cite{z:2002} and Rademacher complexities \cite{b:2001} for the hypothesis space play a key role in this approach.
For pairwise learning, these measures have been used for studying the generalization of specific tasks such as ranking \cite{c:2008x,r:2012} and metric learning \cite{c:2016,y:2019,d:2020}.
Recently, some work also explored the generalization of deep networks with these tasks \cite{h:2023,z:2024}.
Furthermore, generalizations for the pairwise learning framework were studied under various settings, including PL condition \cite{l:2021}, regularized risk minimization \cite{l:2020} and online learning \cite{w:2012,k:2013}.
As compared to the stability analysis, the complexity analysis enjoys the ability of yielding generalization bounds for non-convex objective functions \cite{m:2018,d:2022}.
However, generalization bounds yielded by the uniform convergence approach are inevitably associated with input dimensions \cite{a:2009,f:2016,schliserman2024dimension}, which can be avoided in the stability analysis.

Algorithmic stability serves as an important concept in SLT, which is closely related to learnability and consistency \cite{f:2016,r:2005}.
The basic framework for stability analysis was proposed by \citeauthor{b:2002} \shortcite{b:2002}, where the concept of uniform stability was introduced and then extended to study randomized algorithms \cite{e:2005}.
The power of algorithmic stability for generalization analysis further inspired several other stability measures including uniform argument stability \cite{l:2017-u}, on-average loss stability \cite{s:2010,l:2020,l:2021}, on-average argument stability \cite{l:2020y,deoraoptimization}, locally elastic stability \cite{d:2021,lei2023stability} and Bayes stability \cite{l:2020x}.
While various stability measures were useful for deriving generalization bounds in expectation, applications of uniform stability implied elegant high-probability generalization bounds \cite{v:2019,b:2020,k:2021}.
Furthermore, the stability analysis promoted the study for the generalization of stochastic gradient descent (SGD) effectively \cite{d:2023}, which was considered under the paradigm of pairwise learning \cite{l:2020,l:2021} or pointwise and pairwise learning \cite{w:2023,chen2023stability}.
In contrast to SGD, a more sufficient generalization analysis of RCD is needed under the framework of pairwise learning.
It  provides us guidelines to apply RCD in large-scale optimization problems for pairwise learning.

Other than the approach based on uniform convergence or algorithmic stability, the generalization for pairwise learning was also studied from the perspective of algorithmic robustness \cite{b:2015,c:2016x}, convex analysis \cite{y:2016}, integral operators \cite{f:2016x,g:2017} and information theoretical analysis~\citep{dongtowards2024}.

\section{Preliminaries}\label{se3}

Let $S = \{ z_1,\ldots,z_n \}$ be a set drawn independently from a probability measure $\rho$ defined over a sample space $\mathcal{Z} = \mathcal{X} \times \mathcal{Y}$, where $\mathcal{X}$ is an input space and $\mathcal{Y} \subset \mathbb{R}$ is an output space.
For pairwise learning, our aim is to build a model $h : \mathcal{X} \mapsto \mathbb{R}$ or $h : \mathcal{X} \times \mathcal{X} \mapsto \mathbb{R}$ to simulate the potential mapping lying on $\rho$.
We further assume that the model is parameterized as $h_\mathbf{w}$ and the vector $\mathbf{w}$ belongs to a parameter space $\mathcal{W} \subseteq \mathbb{R}^d$.
As the essential feature of pairwise learning, the nonnegative loss function takes the form of $f : \mathcal{W} \times \mathcal{Z} \times \mathcal{Z} \mapsto \mathbb{R}$.
Since ranking and metric learning are the most popular applications of pairwise learning, we take them as examples here to show how the learning framework involves various learning tasks.
Besides, we present details of AUC maximization below, which is used as the experimental validation for our results.

\begin{example}\label{ex1}
{\rm (Ranking).}
Ranking models usually take the form of $h_{\mathbf{w}} : \mathcal{X} \mapsto \mathbb{R}$.
Given two instances $z = (x,y) , z' = (x',y')$, we adopt the ordering of $h_{\mathbf{w}}(x),h_{\mathbf{w}}(x')$ as the prediction of the ordering for $y,y'$.
As a result, the prediction $h_{\mathbf{w}}(x)-h_{\mathbf{w}}(x')$ and the true ordering $sgn(y-y')$ jointly formulate the approach to measure the performance of models.
The loss function in this problem is further defined as the pairwise formulation of $f(\mathbf{w};z,z') = \phi(sgn(y-y')(h_{\mathbf{w}}(x)-h_{\mathbf{w}}(x')))$.
Here we can choose the logistic loss $\phi(t) = \log(1+\exp(-t))$ or the hinge loss $\phi(t) = \max\left\{1-t,0\right\}$. % as function $\phi$.
\end{example}

\begin{example}\label{ex2}
{\rm (Supervised metric learning).}
For this problem with output space as $\mathcal{Y} = \left\{+1,-1\right\}$, the most usual aim is to learn a Mahalanobis metric $d_{\mathbf{w}}(x,x') = (x-x')^\top\mathbf{w}(x-x')$.
Under the parameter $\mathbf{w} \in \mathbb{R}^{d \times d}$ and the corresponding metric, we hope that the distance metric between two instances is consistent with the similarity of labels.
Let $\phi$ be the logistic or the hinge loss defined in Example \ref{ex1}.
We can formulate this metric learning problem under the framework of pairwise learning by the loss function as $f(\mathbf{w};z,z') = \phi(\tau(y,y')d_{\mathbf{w}}(x,x'))$, where $\tau(y,y') = 1$ if $y = y'$ and $\tau(y,y') = -1$ if $y \neq y'$.
\end{example}

\begin{example}\label{ex3}
{\rm (AUC Maximization).}
AUC score is widely applied to measure the performance of classification models for imbalanced data.
With the binary output space $\mathcal{Y} = \left\{+1,-1\right\}$, it shows the probability that the model $h_{\mathbf{w}} : \mathcal{X} \mapsto \mathbb{R}$ scores a positive instance higher than a negative instance.
Therefore, the loss function for AUC maximization usually takes the form of $f(\mathbf{w};z,z') = g(\mathbf{w}^\top (x-x'))\mathbb{I}_{[y=1,y'=-1]}$,
where $g$ can be chosen in the same way as $\phi$ in Example \ref{ex1} and $\mathbb{I}$ denotes the indicator function.
This demonstrates that AUC maximization also falls into the framework of pairwise learning.
\end{example}

With the pairwise loss function, the population risk is defined as the following
$$
F(\mathbf{w}) = \mathbb{E}_{z_i,z_j\sim\rho}\left[f(\mathbf{w};z_i,z_j)\right],
$$
which can measure the performance of $h_{\mathbf{w}}$ in real applications.
Since $\rho$ is unknown, we consider the empirical risk
$$
F_{S}(\mathbf{w}) = \frac{1}{n(n-1)} \sum\limits_{i,j\in[n]:i \neq j}f(\mathbf{w};z_i,z_j),
$$
where $[n] := \{1,\ldots,n\}$.
Let $\mathbf{w}^{*} = \arg \min_{\mathbf{w} \in \mathcal{W}}F(\mathbf{w})$ and $\mathbf{w}_{S} = \arg \min_{\mathbf{w} \in \mathcal{W}}F_S(\mathbf{w})$.
To approximate the best model $h_{\mathbf{w}^{*}}$, we apply a randomized algorithm $A$ to the training dataset $S$ and get a corresponding output model.
We then use $A(S)$ to denote the parameter of the output model.

Comparing the acquired parameter $A(S)$ and the best parameter $\mathbf{w}^{*}$, the excess risk $F(A(S)) - F(\mathbf{w}^{*})$ can quantify the performance of $A(S)$ appropriately.
We are interested in bounding the excess risk to provide theoretical supports for the practice of learning tasks.
To study the risk adequately, we introduce the following decomposition
\begin{align}\label{eq1}
F(A(S)) & - F(\mathbf{w}^{*}) =  \left[F(A(S)) - F(\mathbf{w}^{*})\right] - [F_{S}(A(S)) \notag \\
& - F_{S}(\mathbf{w}^{*})] + \left[F_{S}(A(S)) - F_{S}(\mathbf{w}^{*})\right].
\end{align}
Taking expectation on both sides of the above equation and noting $\mathbb{E}_{S}\left[F_S(\mathbf{w}^{*})\right] = F(\mathbf{w}^{*})$, we further decompose the excess risk as
\begin{align}\label{eq2}
\mathbb{E}_{S,A}\big[F(A(S)) & - F(\mathbf{w}^{*})\big]
 =  \, \, \mathbb{E}_{S,A}\left[F(A(S)) - F_{S}(A(S))\right] \notag \\
& + \mathbb{E}_{S,A}\left[F_{S}(A(S)) - F_{S}(\mathbf{w}^{*})\right].
\end{align}
The first and the second term on the right-hand side are referred to as estimation error (generalization gap) and optimization error respectively.
We incorporate SLT and optimization theory to control the two errors, respectively.

In this paper, we consider the learning framework below, which combines RCD and pairwise learning.

\begin{definition}\label{de1}
{\rm (RCD for pairwise learning).}
Let $\mathbf{w}_1 \in \mathcal{W}$ be the initial point and $\{\eta_t\}$ be a nonnegative stepsize sequence.
At the $t$-th iteration, we first draw $i_t$ from the discrete uniform distribution over $\{1,\ldots,d\}$ and then update along the $i_t$-th coordinate as
\begin{align}\label{eq3}
\mathbf{w}_{t+1} = \mathbf{w}_t - \eta_t \nabla_{i_t} F_S(\mathbf{w}_t)\mathbf{e}_{i_t},
\end{align}
where $\nabla_{i_t} F_S(\mathbf{w}_t)$ denotes the gradient of the empirical risk w.r.t. to the $i_t$-th coordinate and $\mathbf{e}_{i_t}$ is a vector with the $i_t$-th coordinate being $1$ and other coordinates being $0$.
\end{definition}

Considering the generalization for the above paradigm, we leverage the concept of algorithmic stability to handle the estimation error.
Algorithmic stability shows how algorithms react to perturbations of training datasets.
Various stability measures have been proposed to study the generalization gap in SLT, including uniform stability, argument stability and on-average stability.
Here we introduce the uniform stability and the on-average argument stability, with the latter being particularly useful for generalization analysis in this paper.
It is notable that we follow \citeauthor{l:2020y} \shortcite{l:2020y} in the definition of $\ell_1$ and $\ell_2$ on-average argument stabilities.
The $\ell_1$ on-average argument stability refers to the $\ell_1$-norm of the vector $(\|A(S)-A(S_1)\|_2,\ldots,\|A(S)-A(S_n)\|_2)$, while the $\ell_2$ on-average argument stability refers to the $\ell_2$-norm of this vector.

\begin{definition}\label{de2}
{\rm (Algorithmic Stability).}
Drawing independently from $\rho$, we get the following two datasets
$$
S = \{z_{1},\ldots,z_{n}\} \qquad and \qquad S' = \{z_{1}',\ldots,z_{n}'\}.
$$
We then replace $z_i$ in $S$ with $z_{i}'$ for any $i \in [n]$ and have
$$
S_i = \{z_{1},\ldots,z_{i-1},z_{i}',z_{i+1},\ldots,z_{n}\}.
$$
Let $\mathbf{x} \in \mathbb{R}^{d}$ be a vector of dimension $d$.
Then we denote the $p$-norm $\|\mathbf{x}\|_{p} = (\sum_{i=1}^{d}|\mathbf{x}_{i}|^{p})^{1/p}$
and show several stability measures below.\\
(a) Randomized algorithm $A$ is $\epsilon$-uniformly stable if for any $S,S_i \in \mathcal{Z}^{n}$ the following inequality holds
$$\sup\limits_{z,\tilde{z}}\left[f(A(S),z,\tilde{z}) - f(A(S_i),z,\tilde{z})\right] \leq \epsilon.
$$
(b) We say $A$ is $\ell_1$ on-average argument $\epsilon$-stable if
$$
\mathbb{E}_{S,S',A}\Big[\frac{1}{n}\sum_{i=1}^{n}\|A(S)-A(S_i)\|_{2}\Big]\leq \epsilon.
$$
(c) We say $A$ is $\ell_{2}$ on-average argument $\epsilon$-stable if
$$
\mathbb{E}_{S,S',A}\Big[\frac{1}{n}\sum_{i=1}^{n}\|A(S)-A(S_i)\|_{2}^{2}\Big] \leq \epsilon^2.
$$
\end{definition}

As indicated below, We prepare several necessary assumptions so that relative generalization bounds can be derived effectively.
Assumption \ref{as1} and Assumption \ref{as2} are useful for bounding the on-average argument stability.
Assumption \ref{as3} is mainly applied in the proof of the optimization error.
The other two assumptions show the convexity of the empirical risk, which is the basic condition for the establishment of relative theorems.

\begin{assumption}\label{as1}
For all $(z,z') \in \mathcal{Z} \times \mathcal{Z}$ and $\mathbf{w} \in \mathcal{W}$, the loss function satisfies the $G$-Lipschitz continuity condition as $\| \nabla f(\mathbf{w},z,z')\|_2 \leq G$.
\end{assumption}

\begin{assumption}\label{as2}
For all $(z,z') \in \mathcal{Z} \times \mathcal{Z}$ and $\mathbf{w},\mathbf{w}' \in \mathcal{W}$, the loss function is $L$-smooth as $\|\nabla f(\mathbf{w};z,z') - \nabla f(\mathbf{w'};z,z')\|_{2} \leq L\|\mathbf{w} - \mathbf{w'}\|_{2}$.
\end{assumption}

\begin{assumption}\label{as3}
For any $S$, $F_S$ has coordinate-wise Lipschitz continuous gradients with parameter $\widetilde{L} > 0$, i.e., we have the following inequality for all $\alpha \in \mathbb{R}$, $\mathbf{w} \in \mathcal{W}$, $i \in [d]$
$$
F_S(\mathbf{w} + \alpha \mathbf{e}_{i}) \leq F_S(\mathbf{w}) + \alpha\nabla_{i}F_S(\mathbf{w}) + \widetilde{L}\alpha^{2}/2.
$$
\end{assumption}

\begin{assumption}\label{as4}
$F_S$ is convex for any $S$, i.e., $F_S(\mathbf{w}) - F_S(\mathbf{w}') \geq \left< \mathbf{w}-\mathbf{w}',\nabla F_S(\mathbf{w}') \right>$ holds for all $\mathbf{w},\mathbf{w}' \in \mathcal{W}$.
\end{assumption}

\begin{assumption}\label{as5}
$F_S$ is $\sigma$-strongly convex for any $S$, i.e., the following inequality holds for all $\mathbf{w},\mathbf{w}' \in \mathcal{W}$
$$
F_S(\mathbf{w}) - F_S(\mathbf{w}') \geq \left< \mathbf{w}-\mathbf{w}',\nabla F_S(\mathbf{w}') \right> + \sigma\| \mathbf{w}-\mathbf{w}' \|_{2}^{2}/2,
$$
where $\left< \cdot,\cdot \right>$ denotes the inner product of two vectors.
\end{assumption}

With Definition \ref{de1} and Definition \ref{de2}, we can further quantify stabilities of RCD for pairwise learning.
Then we show connections between the estimation error and stability measures by the following lemma, which is the key to apply algorithmic stability effectively in generalization analysis.
While part (a) of Lemma \ref{le1} is motivated by the case of pointwise learning \cite{h:2016} and derived with the technique similar to \citeauthor{l:2021} \shortcite{l:2021}, part (b) and part (c) are introduced from \citeauthor{l:2021} \shortcite{l:2021} and \citeauthor{l:2020} \shortcite{l:2020} respectively.
In part (c), the base of the natural logarithm takes the symbol as $e$ and $\lceil\alpha\rceil$ means rounding up for $\alpha$.

\begin{lemma}\label{le1}
Let $S,S_i$ be constructed as Definition \ref{de2}.
Then we bound estimation errors with stability measures  below.\\
(a) Let Assumption \ref{as1} hold. Then the estimation error can be bounded by the $\ell_1$ on-average argument stability  below
\begin{multline}\label{eq4}
\mathbb{E}_{S,A} \left[F(A(S)) - F_S(A(S))\right]\\
\leq \frac{2G}{n} \sum\limits_{i=1}^{n} \mathbb{E}_{S,S',A} \left[\| A(S_{i}) - A(S) \|_2\right].
\end{multline}
(b) Let Assumption \ref{as2} hold. Then for any $\gamma >0$ we have the following estimation error bound with the $\ell_2$ on-average argument stability
\begin{align}\label{eq5}
\mathbb{E}_{S,A} & \left[F(A(S)) - F_S(A(S))\right] \leq \frac{L}{\gamma} \mathbb{E}_{S,A}\left[F_S(A(S))\right]\notag\\
& + \frac{2(L+\gamma)}{n} \sum\limits_{i=1}^{n}\mathbb{E}_{S,S',A}\left[\|A(S_i)-A(S)\|_{2}^{2}\right].
\end{align}
(c) Let $n$ denote the sample size of $S$.
Assume for any $S$ and $(z,z') \in \mathcal{Z} \times \mathcal{Z}$, $|f(A(S);z,z')| \leq R$ holds for  $R > 0$.
Suppose that $A$ is $\epsilon$-uniformly-stable and $\delta \in (0,1/e)$, then the following inequality holds with probability at least $1-\delta$
\begin{align}\label{eq6}
|F(A(S)) - & F_S(A(S))|
\leq 4\epsilon + e\Big(12\sqrt{2}R\sqrt{\frac{\log(e/\delta)}{n-1}} \notag \\
& + 48\sqrt{6}\epsilon\lceil\log_{2}(n-1)\rceil\log(e/\delta)\Big).
\end{align}
\end{lemma}

\begin{remark}\label{re1}
While estimation error bound (\ref{eq4}) is established under the Lipschitz continuity condition, (\ref{eq5}) remove this condition based on the $\ell_2$ on-average argument stability measure.
Inequality (\ref{eq5}) holds with the $L$-smoothness of the loss function, which replaces the Lipschitz constant in (\ref{eq4}) by the empirical risk.
Furthermore, if $A$ is $\ell_2$ on-average argument $\epsilon$-stable, we can take $\gamma = \sqrt{L\mathbb{E}_{S,A}\left[F_S(A(S))\right]} / (\sqrt{2}\epsilon)$ in part (b) and get $\mathbb{E}_{S,A}\left[F(A(S)) - F_S(A(S))\right] \leq \sqrt{2L\mathbb{E}_{S,A}\left[F_S(A(S))\right]}\epsilon + 2L\epsilon^2$.
If the empirical risk $\mathbb{E}_{S,A}\left[F_S(A(S))\right] = O(1/n)$, then we further know $\mathbb{E}_{S,A}\left[F(A(S)) - F_S(A(S))\right] = O(\epsilon^2 + \epsilon/\sqrt{n})$, which means the estimation error bound is well dependent on the stability measure $\epsilon$ via the small risk of the output model \cite{h:2016}.
Other than the generalization error in expectation, the link in high probability (\ref{eq6}) presents the convergence rate of $O(n^{-\frac{1}{2}} + \epsilon\sqrt{\log_2(n)} )$ for $\epsilon$-uniformly stable algorithm.
This result is achieved by combining a concentration inequality from \citeauthor{b:2020} \shortcite{b:2020} and the decoupling technique in \citeauthor{l:2020y} \shortcite{l:2020y}.
\end{remark}

Besides the estimation error, we need to tackle the optimization error to achieve complete excess risk bounds.
The optimization error analysis for pointwise learning can be directly extended to pairwise learning since they both use unbiased gradient estimations.
Since pointwise learning and pairwise learning mainly differ in terms of loss structure, Lemma \ref{le2} from pointwise learning also works for pairwise learning.

\begin{lemma}\label{le2}
Let $\{\mathbf{w}_t\}$ be produced by RCD (3) with nonincreasing step sizes $\eta_t \leq 1/\widetilde{L}$. Let Assumptions 3,4 hold, then the following two inequalities holds for any $\mathbf{w} \in \mathcal{W}$
\begin{align}\label{eq7}
\!\!\!\!\mathbb{E}_{A}[F_S & (\mathbf{w}_{t})\!-\!F_S(\mathbf{w})] \leq \frac{d\left( \|\mathbf{w}_{1}\!-\!\mathbf{w}\|_{2}^{2} \!+\! 2\eta_{1}F_{S}(\mathbf{w}_{1}) \right)}{2\sum_{j=1}^{t}\eta_j}
\end{align}
and
\begin{align}\label{eq8}
2\sum\limits_{j=1}^{t}\eta_{j}^{2}\mathbb{E}_A [F_S( & \mathbf{w}_{j}) - F_S(\mathbf{w})] \notag \\
& \leq d\eta_1\| \mathbf{w}_1 - \mathbf{w} \|_{2}^{2} + 2 d \eta_{1}^{2} F_S(\mathbf{w}_1).
\end{align}
Let Assumption 5 hold and $\mathbf{w}_S = \arg \min_{\mathbf{w} \in \mathcal{W}} F_S(\mathbf{w})$, then we have the following inequality
\begin{align}\label{eq9}
\mathbb{E}_{A}[F_S ( & \mathbf{w}_{t+1}) - F_S(\mathbf{w}_{S})] \notag\\ \notag\\
& \quad \leq (1-\eta_{t}\sigma/d)\mathbb{E}_{A}[F_{S}(\mathbf{w}_{t})- F_{S}(\mathbf{w}_{S})].
\end{align}
\end{lemma}

In the arXiv version, Appendix B restates the above two lemmas and prepares some other lemmas.
The proof for part (a) of Lemma \ref{le1} is given in Appendix B.1.
Considering the stability analysis, we introduce the coercivity property of the gradient descent operator in Appendix B.3 \cite{h:2016}.
Then we show the self-bounding property of $L$-smooth functions in Appendix B.4 \cite{s:2010x}, which plays a key role in introducing empirical risks into the $\ell_2$ on-average argument stability.

\section{Main Results}\label{se4}

In this section, we show our results on generalization analysis of RCD for pairwise learning.
For both convex and strongly convex cases, we derive the on-average argument stability bounds and as well as the corresponding excess risk bounds.
Results are organized according to the convexity of the empirical risk.

\subsection{Generalization for Convex Case}\label{se41}

This subsection describes the $\ell_2$ on-average argument stabilities for the convex empirical risk.
Based on stability analysis, we consider generalization bounds in expectation under the setting that applies RCD for pairwise learning.

If the empirical risk is convex and $L$-smooth, then the gradient descent operator enjoys the coercivity property according to \citet{h:2016}.
Since taking expectations for the coordinate descent operator yields the gradient descent operator, the coercivity property is useful to bound the expansiveness of RCD updates in the stability analysis.
With the coercivity property of the coordinate descent operator in expectation, we further incorporate the self-bounding property of $L$-smooth functions to measure the $\ell_2$ on-average argument stability.
Then we handle the estimation error by plugging the stability measure into part (b) of Lemma \ref{le1}.
We finally introduce the optimization error and derive the corresponding excess risk bound.
The proof is given in Appendix C of the arXiv version.

\begin{theorem}\label{th3}
Let Assumptions \ref{as2}, \ref{as3}, \ref{as4} hold. Let $\{\mathbf{w}_{t}\}$, $\{\mathbf{w}_{t}^{(i)}\}$ be produced by (\ref{eq3}) with $\eta_t \leq 1/L$ based on $S$ and $S_i$ respectively.
Then the $\ell_2$ on-average argument stability satisfies
\begin{align}\label{eq10}
&\frac{1}{n}\sum\limits_{i=1}^{n}\mathbb{E}_{S,S',A}\Big[\|\mathbf{w}_{t+1} - \mathbf{w}_{t+1}^{(i)}\|_{2}^{2}\Big] \notag\\ &\leq\frac{128L}{n^2d}(\frac{t}{d}+1)\sum\limits_{j=1}^{t}
\eta_{j}^{2}\mathbb{E}_{S,A}[F_S(\mathbf{w}_j)].
\end{align}
Assume that the nonincreasing step size sequence $\{ \eta_t \}$ satisfies $\eta_t \leq 1/\widetilde{L}$.
Then, for any $\gamma \geq 0$, we have
\begin{align}\label{eq11}
& \mathbb{E}_{S,A}\left[F(\mathbf{w}_{T}) - F(\mathbf{w}^{*})\right]\notag\\
& = O\left(\frac{d(1+L\gamma^{-1})}{\sum_{t=1}^{T}\eta_t}
+\frac{L(L+\gamma)(T+d)} {n^2d}\right) \notag\\
& + O\left(\frac{L}{\gamma}+\frac{L(L+\gamma)(T+d)}{n^2d^2}
\sum\limits_{t=1}^{T}\eta_{t}^{2} \right) \times F(\mathbf{w}^{*}).
\end{align}
Furthermore, for a constant step size as $\eta_t \equiv \eta$, we choose $T \asymp n^{\frac{1}{2}}dL^{-\frac{1}{2}}$ and get
\begin{align}\label{eq12}
\mathbb{E}_{S,A}\left[F(\mathbf{w}_{T}) - F(\mathbf{w}^{*})\right] = O\Big( \sqrt{\frac{L}{n}}\Big).
\end{align}
Assuming that $F(\mathbf{w}^{*}) = O(Ln^{-1})$, we choose $T \asymp ndL^{-1}$ to give
\begin{align}\label{eq13}
\mathbb{E}_{S,A}\left[F(\mathbf{w}_{T}) - F(\mathbf{w}^{*})\right]
= O\left(\frac{L}{n} \right).
\end{align}
\end{theorem}

\begin{remark}\label{re2}
For pairwise learning, Eq.~(\ref{eq10}) shows that RCD enjoys the $\ell_2$ on-average argument stability is of the order of $O\big( L(t+d) \sum_{j=1}^{t} \eta_{j}^{2}\mathbb{E}_{S,A}[F_S(\mathbf{w}_j)] / (n^2 d^2) \big).$
This bound means that the output model of RCD becomes more and more stable with the sample size increasing or the number of iterations decreasing.
In further detail, since the estimation error can be bounded by the stability bound according to (\ref{eq5}), decreasing the number of iterations is beneficial to controlling the estimation error.
However, increasing the number of iterations corresponds to the optimization process, which is the key to control the optimization error.
As a result, the early stopping strategy is adopted to balance the  estimation and optimization for a good generalization.
\end{remark}

\begin{remark}\label{re3}
Fixing $F(\mathbf{w}^*)$, we choose an appropriate number of iterations for the excess risk bound (\ref{eq12}).
Besides, we incorporate $F(\mathbf{w}^{*})$ into the excess risk bound and get the convergence rate (\ref{eq13}).
It is obvious that (\ref{eq13}) exploits the low noise setting to yield the optimistic bound \cite{s:2010x}.
Furthermore, since RCD updates in expectation is closely related to the gradient descent operator, we consider the batch setting \cite{nx:2023} and find that results here are identical to those of full-batch GD \cite{n:2023}.
Turning to SGD for pairwise learning \cite{l:2021}, the $\ell_2$ on-average argument stability takes a slower rate as $O(1/n)$ under the same setting.
The excess risk bound can achieve the rate of $(1/\sqrt{n})$ in a general setting with $\eta_t = \eta \asymp 1/\sqrt{T}$ and $T \asymp n$.
With the low noise setting $F(\mathbf{w}^{*}) = O(n^{-1})$, the optimistic bound $O(1/n)$ is also derived.
\end{remark}

\begin{figure*}[t]
    \captionsetup{labelformat=empty}
    \captionsetup{aboveskip=3pt}
    \centering

    \begin{minipage}{0.21\textwidth}
        \includegraphics[width=\linewidth]{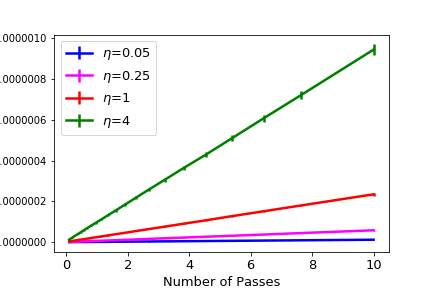}
        \includegraphics[width=\linewidth]{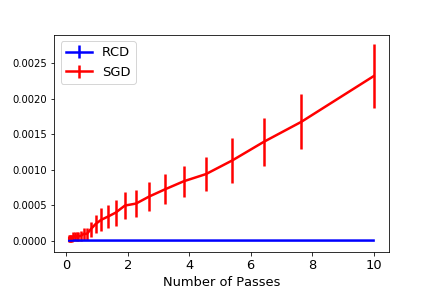}
        \caption{(a) hinge for a3a} \label{fig1}
    \end{minipage}\hfill
    \begin{minipage}{0.21\textwidth}
        \includegraphics[width=\linewidth]{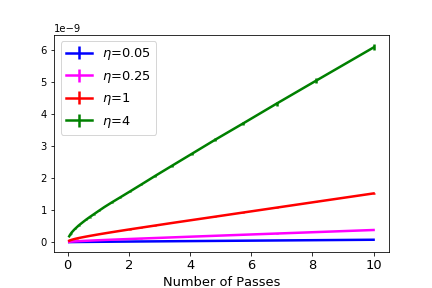}
        \includegraphics[width=\linewidth]{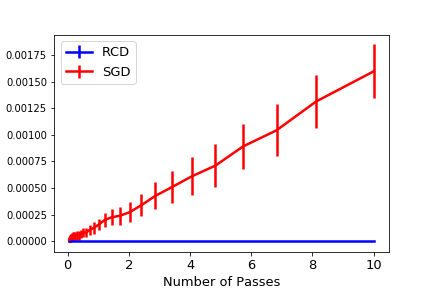}
        \caption{(b) hinge for gisette}
    \end{minipage}\hfill
    \begin{minipage}{0.21\textwidth}
        \includegraphics[width=\linewidth]{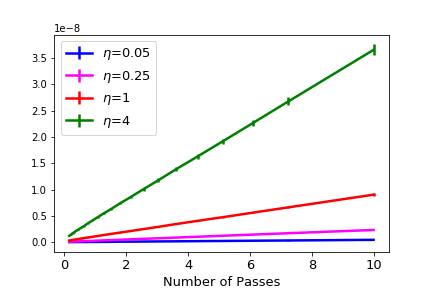}
        \includegraphics[width=\linewidth]{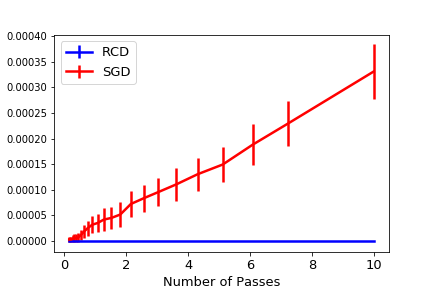}
        \caption{(c) hinge for madelon}
    \end{minipage}\hfill
    \begin{minipage}{0.21\textwidth}
        \includegraphics[width=\linewidth]{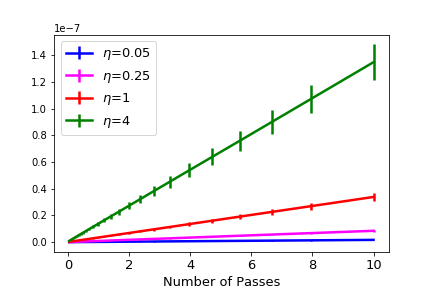}
        \includegraphics[width=\linewidth]{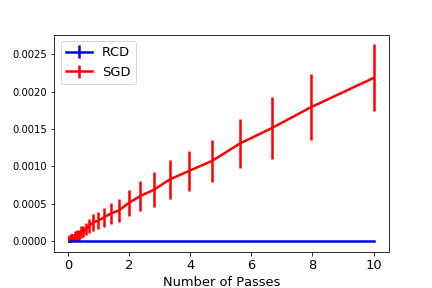}
        \caption{(d) hinge for usps}
    \end{minipage}

    \begin{minipage}{0.21\textwidth}
        \includegraphics[width=\linewidth]{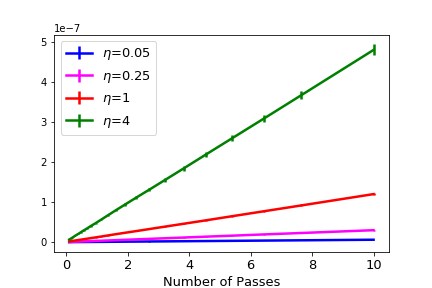}
        \includegraphics[width=\linewidth]{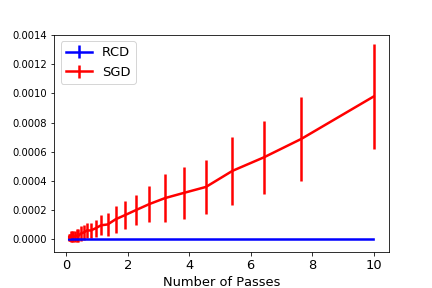}
        \caption{(e) logistic for a3a}
    \end{minipage}\hfill
    \begin{minipage}{0.21\textwidth}
        \includegraphics[width=\linewidth]{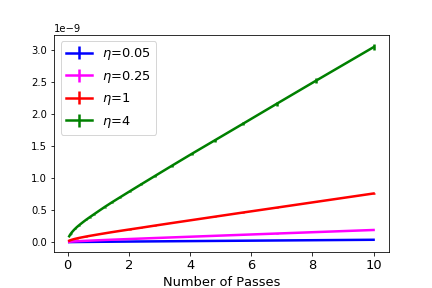}
        \includegraphics[width=\linewidth]{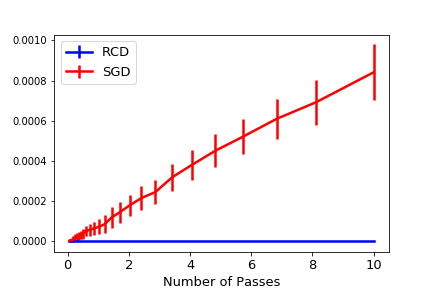}
        \caption{(f) logistic for gisette}
    \end{minipage}\hfill
    \begin{minipage}{0.21\textwidth}
        \includegraphics[width=\linewidth]{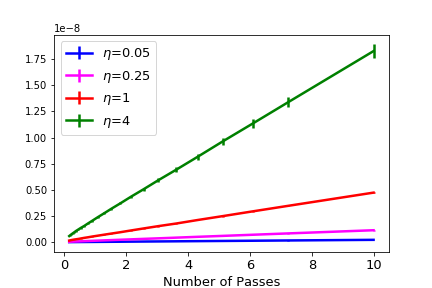}
        \includegraphics[width=\linewidth]{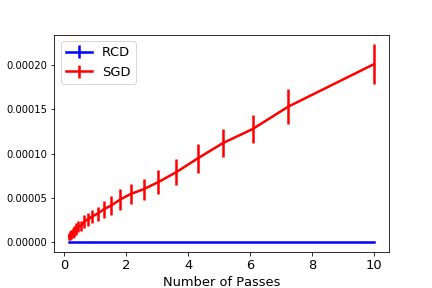}
        \caption{(g) logistic for madelon}
    \end{minipage}\hfill
    \begin{minipage}{0.21\textwidth}
        \includegraphics[width=\linewidth]{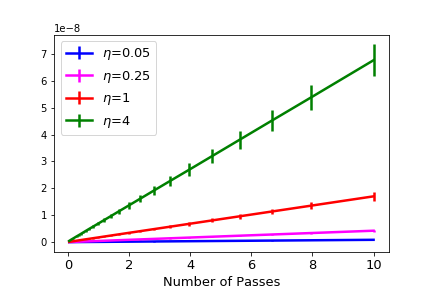}
        \includegraphics[width=\linewidth]{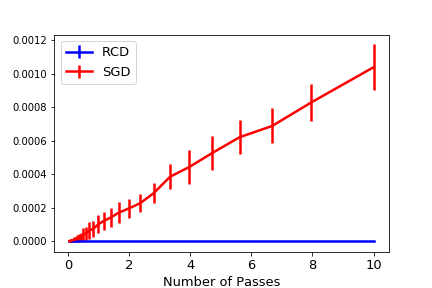}
        \caption{(h) logistic for usps}
    \end{minipage}

    \caption{Figure 1: Euclidean distance $\Delta_t$ as a function of the number of passes for the hinge loss.}
\end{figure*}

\subsection{Generalization for Strongly Convex Case}\label{se42}

This subsection presents generalization analysis of RCD for pairwise learning in a strongly convex setting.
In the strongly convex case of pointwise learning \cite{p:2021}, the $\ell_2$ on-average argument stability and the corresponding generalization bound were not taken into consideration.
Therefore, we not only measure the stability here but also derive the excess risk bound for the strongly convex empirical risk.
We show the proof in Appendix D of the arXiv version.

\begin{theorem}\label{th4}
Let Assumptions \ref{as2}, \ref{as3}, \ref{as5} hold.
Let $\{\mathbf{w}_{t}\}$, $\{\mathbf{w}_{t}^{(i)}\}$ be produced by (\ref{eq3}) with $\eta_t \leq \beta/L$ for any $\beta \in (0,1)$ based on $S$ and $S_i$, respectively.
Then the $\ell_2$ on-average  argument stability is
\begin{align}\label{eq14} &\frac{1}{n}\sum\limits_{i=1}^{n}\mathbb{E}_{S,S',A}\left[\|\mathbf{w}_{t+1} - \mathbf{w}_{t+1}^{(i)}\|_{2}^{2}\right] \notag\\
&\leq \frac{128L}{n^2d} \sum\limits_{j=1}^{t}
\Big(\frac{t}{d}\prod\limits_{k=j+1}^{t}
\big(1-\frac{2\eta_k(1-\beta)(n-2)\sigma}{nd}\big)^2 \notag\\
&+ \prod\limits_{k=j+1}^{t}
\big(1-\frac{2\eta_k(1-\beta)(n-2)\sigma}{nd}\big)\Big)
\eta_{j}^{2}\mathbb{E}_{S,A}[F_S(\mathbf{w}_j)].
\end{align}
Let step sizes be fixed as $\eta_t \equiv \eta \leq 1 / \widetilde{L}$. For any $\gamma \geq 0$, we develop the excess risk bound as
\begin{align}\label{eq15}
&\mathbb{E}_{S,A}\left[F(\mathbf{w}_{T+1}) - F_S(\mathbf{w}_S)\right]\notag\\
&= O\left(\big(1+\frac{L}{\gamma}\big)(1-\eta\sigma/d)^T + \frac{L(d+T)(L+\gamma)}{(n-2)^2\sigma^2(1-\beta)^2}\right)  \notag\\
& + O\left(\frac{L}{\gamma} + \frac{L(d+T)(L+\gamma)}{(n-2)^2\sigma^2(1-\beta)^2}\right)\times \mathbb{E} [ F_S(\mathbf{w}_S) ].
\end{align}
Choosing $T \asymp d\sigma^{-1}\log(n\sigma L^{-1})$ yields
\begin{align}\label{eq16}
\mathbb{E}_{S,A}  [F(\mathbf{w}_{T+1})  - F_S(\mathbf{w}_S)]
= O\left(\frac{Ld^{\frac{1}{2}}}{n\sigma^{\frac{3}{2}}}
\sqrt{\log\big(\frac{n\sigma}{L}\big)}\right).
\end{align}
\end{theorem}

\begin{remark}\label{re4}
As shown in (\ref{eq14}), the stability measure involves a weighted sum of empirical risks.
This demonstrates that low risks of output models can improve the stability along the training process.
The measure also shows information including the convexity parameter $\sigma$ and learning rates $\eta_j$ which are closely associated with the interplay between RCD and training datasets.
Furthermore, the strong convexity of the empirical risk obviously leads to a better stability as compared to the convex case (\ref{eq10}).
\end{remark}

\begin{remark}\label{re5}
The convergence rate (\ref{eq16}) gives the choice of $T$ to balance the estimation and optimization.
Indeed, the optimal convergence rate lies between $O(1/(n\sigma))$ and $O\big(\sqrt{\log (n\sigma)}/(n\sigma^{\frac{3}{2}})\big)$, for which the corresponding choices of $T$ are smaller than that we give.
It is notable that the approximate optimal rate here almost matches the optimistic bound (\ref{eq13}).
Besides, the strong convexity promotes the fast computing $T \asymp \log (n)$ as compared to $T \asymp n$ under the convex case.
Results here are the same as those for full-batch GD \cite{n:2023}, which can verify the theorem since the expectation for RCD leads to the gradient descent operator.
Considering SGD, \citeauthor{l:2021} \shortcite{l:2021} present the generalization bounds for pairwise learning.
Under the same setting of smoothness and strong convexity, SGD achieves the excess risk bound $O(1/(n\sigma))$.
However, the convergence rate of SGD requires the number of iterations as $T \asymp O(n/\sigma)$ and a small $F(\mathbf{w}^{*})$.
\end{remark}

\section{ Experimental Verification}\label{se5}

In this section, we choose the example of AUC maximization to verify the theoretical results on stability measures.
Results are shown in Figure \ref{fig1}.

Here we choose the hinge loss and the logistic loss.
We consider the datasets from LIBSVM \cite{c:2011} and measure the stability of RCD on these datasets, whose details are presented in Appendix E of the arXiv version.
We follow the settings of SGD for pairwise learning \cite{l:2021} and compare the results of RCD and SGD.
In each experiment, we randomly choose $80$ percents of each dataset as the training set $S$.
Then we perturb a a signal example of $S$ to construct the neighboring dataset $S'$.
We apply RCD or SGD to $S,S'$ and get two iterate sequences, with which we plot the Euclidean distance $\Delta_t =  \| \mathbf{w}_t - \mathbf{w}_t'\|_2$ for each iteration.
While the learning rates are set as $\eta_t = \eta/\sqrt{T}$ with $\eta \in \left\{ 0.05, 0.25, 1, 4 \right\}$ for RCD, we only compare RCD and SGD under the setting of $\eta = 0.05$.
Letting $n$ be the sample size, we report $\Delta_t$ as a function of $T/n$ (the number of passes).
We repeat the experiments $100$ times, and consider the average and the standard deviation.

Since both loss functions that we choose are convex, the following discussions are based on the theorems in convex case.
Considering the comparison between SGD and RCD, while the term $(T/n)^2$ dominates the convergence rates of stability bounds for SGD according to \citeauthor{l:2021} \shortcite{l:2021}, the on-average argument stability bound (\ref{eq10}) for RCD takes the order of $O(T/n^2)$.
The experimental results for the comparison are consistent with the theoretical stability bounds.
Furthermore, the Euclidean distance under the logistic loss is significantly smaller than that under the hinge loss, which is consistent with the the discussions of \citeauthor{l:2021} \shortcite{l:2021} for smooth and nonsmooth problems.

\section{Conclusion}\label{se6}

In this paper, we study the generalization performance of RCD for pairwise learning.
We measure the on-average argument stability develop the corresponding excess risk bound.
Results for the convex empirical risk show us how the early-stopping strategy can balance estimation and optimization.
The excess risk bounds enjoy the convergence rates of $O(1/\sqrt{n})$ and $O(\sqrt{\log(n)} / n)$ under the convex and strongly convex cases, respectively.
Furthermore, results under the convex case exploit different noise settings to explore better generalizations.
While the low noise setting $F(\mathbf{w}^{*}) = O(1/n)$ improves the convergence rates from $O(1/\sqrt{n})$ to $O(1/n)$ for convex objective functions, the strong convexity allows the almost optimal convergence rate of $O(1/n)$ with a significantly faster computing as $T \asymp \log (n)$.

There remain several questions for further investigation.
Explorations under the nonparametric or the non-convex case are important for extending the applications of RCD.
RCD for the specific large-scale matrix optimization also deserves a fine-grained generalization analysis.

\onecolumn
\setcounter{definition}{0}
\setcounter{assumption}{0}
\setcounter{theorem}{0}
\setcounter{equation}{0}

\section*{Appendix for  ``Stability-Based Generalization Analysis of Randomized Coordinate Descent
 for Pairwise Learning''}

\section{A \quad Definitions and Assumptions}

\subsection{A.1 \, RCD for Pairwise Learning}

\begin{definition}\label{de1}
{\rm \textbf{(restated)}}.
{\rm (RCD for pairwise learning).}
To present the pairwise learning paradigm clearly, we show the corresponding empirical risk as below
\begin{align}\label{eq1}
F_{S}(\mathbf{w}) = \frac{1}{n(n-1)} \sum\limits_{i,j\in[n]:i \neq j}f(\mathbf{w};z_i,z_j),
\end{align}
where the pairwise loss function $f : \mathcal{W} \times \mathcal{Z} \times \mathcal{Z} \mapsto \mathbb{R}$ measures the model performance on instance pairs and the training dataset is given in Definition \ref{de2}.
This empirical risk can further lead to the following randomized coordinate descent method for pairwise learning.
Let $\mathbf{w}_1 \in \mathcal{W}$ be the initial point and $\{\eta_t\}$ be a nonnegative stepsize sequence.
At the t-th iteration, we draw $i_t$ from the discrete uniform distribution over $\{1,\ldots,d\}$ and update along the $i_t$-th coordinate as
\begin{align}\label{eq2}
\mathbf{w}_{t+1} = \mathbf{w}_t - \eta_t \nabla_{i_t} F_S(\mathbf{w}_t)\mathbf{e}_{i_t},
\end{align}
where $\nabla_{i_t} F_S(\mathbf{w}_t)$ denotes the gradient of the empirical risk w.r.t. to the $i_t$-th coordinate and $\mathbf{e}_{i_t}$ is a vector with the $i_t$-th coordinate being $1$ and other coordinates being $0$.
\end{definition}

\subsection{A.2 \, Descriptions for Datasets}

\begin{definition}\label{de2}
{\rm \textbf{(restated)}}.
Drawing independently from $\rho$, we get the following two datasets
$$
S = \{z_{1},\ldots,z_{n}\} \qquad and \qquad S' = \{z_{1}',\ldots,z_{n}'\}.
$$
We remove $z_i$ from $S$ for any $i \in [n]$ to get
$$
S_{-i} = \{z_{1},\ldots,z_{i-1},z_{i+1},\ldots,z_{n}\}.
$$
We replace $z_i$ in $S$ with $z_{i}'$ for any $i \in [n]$ and have
$$
S_i = \{z_{1},\ldots,z_{i-1},z_{i}',z_{i+1},\ldots,z_{n}\}.
$$
Furthermore, we define another dataset as below
$$
S_{i,j} = \{z_{1},\ldots,z_{i-1},z_{i}',z_{i+1},\ldots,z_{j-1},z_{j}',z_{j+1},\ldots,z_{n}\}.
$$
\end{definition}

\subsection{A.3 \, Some Assumptions}

\begin{assumption}\label{as1}
For all $z,z' \in \mathcal{Z}$ and $\mathbf{w} \in \mathcal{W}$, the loss function is $G$-Lipschitz continuous as $\| \nabla f(\mathbf{w},z,z')\|_2 \leq G$.
\end{assumption}

\begin{assumption}\label{as2}
For all $z,z' \in \mathcal{Z}$ and $\mathbf{w},\mathbf{w}' \in \mathcal{W}$, the loss function is $L$-smooth as
$$
\|\nabla f(\mathbf{w};z,z') - \nabla f(\mathbf{w'};z,z')\|_{2} \leq L\|\mathbf{w} - \mathbf{w'}\|_{2}.
$$
\end{assumption}

\begin{assumption}\label{as3}
For any $S$, $F_S$ has coordinate-wise Lipschitz continuous gradients with parameter $\widetilde{L} > 0$, i.e., we have the following inequality for all $\alpha \in \mathbb{R}$, $\mathbf{w} \in \mathcal{W}$, $i \in [d]$
$$
F_S(\mathbf{w} + \alpha \mathbf{e}_{i}) \leq F_S(\mathbf{w}) + \alpha\nabla_{i}F_S(\mathbf{w}) + \widetilde{L}\alpha^{2}/2.
$$
\end{assumption}

\begin{assumption}\label{as4}
$F_S$ is convex for any $S$, i.e., $F_S(\mathbf{w}) - F_S(\mathbf{w}') \geq \left< \mathbf{w}-\mathbf{w}',\nabla F_S(\mathbf{w}') \right>$ holds for all $\mathbf{w},\mathbf{w}' \in \mathcal{W}$.
\end{assumption}

\begin{assumption}\label{as5}
$F_S$ is $\sigma$-strongly convex for any $S$, i.e., the following inequality holds for all $\mathbf{w},\mathbf{w}' \in \mathcal{W}$
$$
F_S(\mathbf{w}) - F_S(\mathbf{w}') \geq \left< \mathbf{w}-\mathbf{w}',\nabla F_S(\mathbf{w}') \right> + \sigma\| \mathbf{w}-\mathbf{w}' \|_{2}^{2}/2,
$$
where $\left< \cdot,\cdot \right>$ denotes the inner product of two vectors.
\end{assumption}

\section{B \quad Necessary Lemmas}

\subsection{B.1 \, Connections between the Generalization Error and the On-average Argument Stability}

\begin{lemma}\label{le1}
{\rm \textbf{(restated)}}.
Let $S,S',S_i,S_{i,j}$ be defined as Section {\rm A.2}.\\
(a) Let Assumption \ref{as1} hold. Then the estimation error can be bounded by the $\ell_1$ on-average argument stability as below
\begin{align}\label{eq3}
\mathbb{E}_{S,A} \left[F(A(S)) - F_S(A(S))\right]
\leq \frac{2G}{n} \sum\limits_{i=1}^{n} \mathbb{E}_{S,S',A} \left[\| A(S_{i}) - A(S) \|_2\right].
\end{align}
(b) Let Assumption \ref{as2} hold. Then for any $\gamma >0$ we have the estimation error bound by the $\ell_2$ on-average argument stability
\begin{align*}
\mathbb{E}_{S,A}  \left[F(A(S)) - F_S(A(S))\right]
\leq \frac{L}{\gamma} \mathbb{E}_{S,A}\left[F_S(A(S))\right]
+ \frac{2(L+\gamma)}{n}
\sum\limits_{i=1}^{n}\mathbb{E}_{S,S',A}\left[\|A(S_i)-A(S)\|_{2}^{2}\right].
\end{align*}
(c) Let $n$ denote the sample size of $S$.
Assume for any $S$ and $(z,z') \in \mathcal{Z} \times \mathcal{Z}$, $|f(A(S);z,z')| \leq R$ holds for some $R > 0$.
Suppose that $A$ is $\epsilon$-uniformly-stable and $\delta \in (0,1/e)$. Then the following inequality holds with probability at least $1-\delta$
\begin{align*}
|F(A(S)) - & F_S(A(S))|
\leq 4\epsilon + e\Big(12\sqrt{2}R\sqrt{\frac{\log(e/\delta)}{n-1}}
+ 48\sqrt{6}\epsilon\lceil\log_{2}(n-1)\rceil\log(e/\delta)\Big).
\end{align*}
\end{lemma}

\begin{proof}
Part (b) and Part (c) are introduced from \citeauthor{l:2021} \shortcite{l:2021} and \citeauthor{l:2020} \shortcite{l:2020}, respectively.
We only show the proof for part (a) here.
According to the symmetry between $z_i,z_j$ and $z_i',z_j'$, we have
\begin{align*}
\mathbb{E}_{S,A}[F(A(S)) - F_S(A(S))]
&= \frac{1}{n(n-1)} \sum\limits_{i,j \in [n] : i \neq j} \mathbb{E}_{S,S',A}\left[F(A(S_{i,j})) - F_S(A(S))\right] \notag \\
&= \frac{1}{n(n-1)} \sum\limits_{i,j \in [n] : i \neq j} \mathbb{E}_{S,S',A} [f(A(S_{i,j});z_i,z_j) - f(A(S);z_i,z_j)],
\end{align*}
where we used $\mathbb{E}_{z_i,z_j}[f(A(S_{i,j});z_i,z_j)]=F(A(S_{i,j}))$ since $z_i,z_j$ are independent of $A(S_{i,j})$.
Then we apply Assumption \ref{as1} in the above equation and get
\begin{align*}
\left| \mathbb{E}_{S,A}[F(A(S)) - F_S(A(S))] \right|
&\leq \frac{G}{n(n-1)} \sum\limits_{i,j \in [n] : i \neq j} \mathbb{E}_{S,S',A} \left[\| A(S_{i,j}) - A(S) \|_2\right] \\
&\leq \frac{G}{n(n-1)} \sum\limits_{i,j \in [n] : i \neq j} \mathbb{E}_{S,S',A} [\| A(S_{i,j}) - A(S_i) \|_2 + \| A(S_{i}) - A(S) \|_2 ] \\
&= \frac{2G}{n} \sum\limits_{i=1}^{n} \mathbb{E}_{S,S',A} \Big[\| A(S_{i}) - A(S) \|_2\Big],
\end{align*}
where in the last step we have used
$$
\mathbb{E}_{S,S',A}\left[\| A(S_{i,j}) - A(S_i) \|_2\right] = \mathbb{E}_{S,S',A}\left[\| A(S_{j}) - A(S) \|_2\right].
$$
This completes the proof for inequality (\ref{eq3}).
\end{proof}

\subsection{B.2 \, Restatement of Optimization Error Bounds in Expectation}

\begin{lemma}\label{le2}
{\rm \textbf{(restated)}}.
Let $\{\mathbf{w}_t\}$ be produced by RCD with nonincreasing step sizes $\eta_t \leq 1/\widetilde{L}$.
Let Assumptions \ref{as3},\ref{as4} hold, then we have the two following inequalities
\begin{align}\label{eq4}
\mathbb{E}_{A}[F_S & (\mathbf{w}_{t})-F_S(\mathbf{w})]
\leq \frac{d}{2\sum_{j=1}^{t}\eta_j}\left( \|\mathbf{w}_{1}-\mathbf{w}\|_{2}^{2} + 2\eta_{1}F_{S}(\mathbf{w}_{1}) \right),
\end{align}
and
\begin{align}\label{eq5}
2\sum\limits_{j=1}^{t}\eta_{j}^{2}\mathbb{E}_A [F_S( & \mathbf{w}_{j}) - F_S(\mathbf{w})]
\leq d\eta_1\| \mathbf{w}_1 - \mathbf{w} \|_{2}^{2} + 2 d \eta_{1}^{2} F_S(\mathbf{w}_1).
\end{align}
Let Assumption 5 hold and $\mathbf{w}_S = \arg \min_{\mathbf{w} \in \mathcal{W}} F_S(\mathbf{w})$, we further have the inequality as below
\begin{align}\label{eq6}
\mathbb{E}_{A}[F_S ( & \mathbf{w}_{t+1}) - F_S(\mathbf{w}_{S})]
\leq (1-\eta_{t}\sigma/d)\mathbb{E}_{A}[F_{S}(\mathbf{w}_{t})- F_{S}(\mathbf{w}_{S})].
\end{align}
\end{lemma}

Since the structure of loss functions does not matter for the optimization error, this lemma from the case of pointwise learning \cite{p:2021} also applies to the optimization error for pairwise learning.

\subsection{B.3 \, The Coercivity Property of Gradient Descent Operators}

\begin{lemma}\label{le3}
Let $g : \mathbb{R}^d \mapsto \mathbb{R}$ be convex and $L$-smooth. Then the following inequality holds
\begin{align}\label{eq7}
\left< \mathbf{w} - \mathbf{w}' , \nabla g(\mathbf{w}) - \nabla g(\mathbf{w}') \right>
\geq \frac{1}{L} \| \nabla g(\mathbf{w}) - \nabla g(\mathbf{w}') \|_{2}^{2}
\end{align}
Furthermore, if $g$ is $\sigma$-strongly convex, then for any $\beta \in (0,1)$ we have
\begin{align}\label{eq8}
\left< \mathbf{w} - \mathbf{w}' , \nabla g(\mathbf{w}) - \nabla g(\mathbf{w}') \right>
\geq \frac{\beta}{L} \| \nabla g(\mathbf{w}) - \nabla g(\mathbf{w}') \|_{2}^{2}
+ (1-\beta) \sigma \| \mathbf{w} - \mathbf{w}' \|_{2}^{2}.
\end{align}
\end{lemma}

Lemma \ref{le3} is according to \citeauthor{h:2016} \shortcite{h:2016}.
Considering the expectation of the randomized coordinate in iterate operator (\ref{eq2}), we further know that the expansion for $\| \mathbf{w}_{t+1} - \mathbf{w}_{t+1}^{(i)} \|_{2}^{2}$ is closely related to the gradient descent operator.
Therefore, this lemma plays a key role in the stability analysis.

\subsection{B.4 \, The Self-bounding Property for $L$-smooth Functions}
We introduce Lemma \ref{le4} from \citeauthor{s:2010x} \shortcite{s:2010}, which demonstrates the gradients of $L$-smooth functions can be bounded by the function values.
The property is the key to remove the Lipschitz continuity assumption on loss functions.
This lemma shows that the gradients of $L$-smooth functions can be bounded by the function values.
\begin{lemma}\label{le4}
Let $g : \mathcal{W} \mapsto \mathbb{R}$ be a nonnegative and $L$-smooth function.
Then for all $\mathbf{w} \in \mathcal{W}$ we have
$$
\| \nabla g(\mathbf{w}) \|_{2}^{2} \leq 2Lg(\mathbf{w}).
$$
\end{lemma}

\section{C \quad Proof for Convex Case}

\begin{theorem}\label{th2}
{\rm \textbf{(restated)}}.
Let Assumptions \ref{as2},\ref{as3},\ref{as4} hold. Let $\{\mathbf{w}_{t}\}$, $\{\mathbf{w}_{t}^{(i)}\}$ be produced by (\ref{eq1}) with $\eta_t \leq 1/L$ based on $S$ and $S_i$ respectively.
Then we have the $\ell_2$ on-average argument stability as below
\begin{align}\label{eq9}
&\frac{1}{n}\sum\limits_{i=1}^{n}\mathbb{E}_{S,S',A}\left[\|\mathbf{w}_{t+1} - \mathbf{w}_{t+1}^{(i)}\|_{2}^{2}\right] \leq\frac{128L}{n^2d}(\frac{t}{d}+1)\sum\limits_{j=1}^{t}
\eta_{j}^{2}\mathbb{E}_{S,A}[F_S(\mathbf{w}_j)].
\end{align}
Assume that the nonincreasing step size sequence $\{ \eta_t \}$ satisfies $\eta_t \leq 1/\widetilde{L}$.
The excess risk bound can be developed as below.
For any $\gamma > 0$, we have the rate
\begin{align}\label{eq10}
\mathbb{E}_{S,A}\left[F(\mathbf{w}_{T}) - F(\mathbf{w}^{*})\right]
= O\left(\frac{d(1+L\gamma^{-1})}{\sum_{t=1}^{T}\eta_t}
+\frac{L(L+\gamma)(T+d)} {n^2d}\right)
+ O\left(\frac{L}{\gamma}+\frac{L(L+\gamma)(T+d)}{n^2d^2}
\sum\limits_{t=1}^{T}\eta_{t}^{2} \right) \times F(\mathbf{w}^{*}),
\end{align}
from which we further consider the convergence rate.
Let the step sizes be fixed as $\eta$.
We can choose $T \asymp n^{\frac{1}{2}}dL^{-\frac{1}{2}}$ and get
\begin{align}\label{eq11}
\mathbb{E}_{S,A}\left[F(\mathbf{w}_{T}) - F(\mathbf{w}^{*})\right] = O\Big( \sqrt{\frac{L}{n}} \Big).
\end{align}
Assuming that $F(\mathbf{w}^{*}) = O(Ln^{-1})$, we can choose $T = O(ndL^{-1})$ to give
\begin{align}\label{eq12}
\mathbb{E}_{S,A}\left[F(\mathbf{w}_{T}) - F(\mathbf{w}^{*})\right]
= O\left(\frac{L}{n} \right).
\end{align}
\end{theorem}

\begin{proof}
To avoid the tricky structure $\sum_{k \neq i}\big[ \nabla f(\mathbf{w}_{t};z_{k},z_{i}) - \nabla f(\mathbf{w}_{t};z_{k},z_{i}') \big] + \sum_{l \neq i}\big[\nabla f(\mathbf{w}_{t};z_{i},z_{l}) - \nabla f(\mathbf{w}_{t};z_{i}',z_{l}) \big]$ in the following derivation, we consider the decomposition $F_{S}(\mathbf{w}) = F_{S_{-i}}(\mathbf{w}) + \frac{1}{n(n-1)}\sum_{l \neq i} \left[f(\mathbf{w};z_i,z_l) + f(\mathbf{w};z_l,z_i)\right]$ and $F_{S_i}(\mathbf{w}) = F_{S_{-i}}(\mathbf{w}) + \frac{1}{n(n-1)}\sum_{l \neq i} \left[f(\mathbf{w};z_i',z_l) + f(\mathbf{w};z_l,z_i')\right]$, where $F_{S_{-i}}(\mathbf{w})$ is defined as
$$
F_{S_{-i}}(\mathbf{w}) = \frac{1}{n(n-1)}\sum\limits_{j,j'\in [n] / i : j \neq j'}f(\mathbf{w};z_j,z_{j'}).
$$
To begin the analysis for the $\ell_2$ on-average argument stability, we handle $\|\mathbf{w}_{t+1} - \mathbf{w}_{t+1}^{(i)}\|_{2}^{2}$ as below
\begin{align}\label{eq13}
&\mathbb{E}_{i_t}[\|\mathbf{w}_{t+1} - \mathbf{w}_{t+1}^{(i)}\|_{2}^{2}]
= \mathbb{E}_{i_t}[\|\mathbf{w}_{t} - \eta_t \nabla_{i_t}F_{S}(\mathbf{w}_{t})\mathbf{e}_{i_t}
- \mathbf{w}_{t}^{(i)} + \eta_t\nabla_{i_t}F_{S_i}(\mathbf{w}_{t}^{(i)})\mathbf{e}_{i_t}\|_{2}^{2}] \notag\\
&= \| \mathbf{w}_t - \mathbf{w}_{t}^{(i)} \|_{2}^{2}
+ \eta_{t}^{2}\mathbb{E}_{i_t}[\| \nabla_{i_t}F_{S}(\mathbf{w}_{t})\mathbf{e}_{i_t} - \nabla_{i_t}F_{S_i}(\mathbf{w}_{t}^{(i)})\mathbf{e}_{i_t} \|_{2}^{2}]
- 2\eta_t \mathbb{E}_{i_t}\big< \mathbf{w}_t - \mathbf{w}_{t}^{(i)} , \nabla_{i_t}F_{S}(\mathbf{w}_{t})\mathbf{e}_{i_t} - \nabla_{i_t}F_{S_i}(\mathbf{w}_{t}^{(i)})\mathbf{e}_{i_t} \big>
\notag\\
&= \| \mathbf{w}_t - \mathbf{w}_{t}^{(i)} \|_{2}^{2}
+ \frac{\eta_{t}^{2}}{d}\| \nabla F_{S}(\mathbf{w}_{t}) - \nabla F_{S_i}(\mathbf{w}_{t}^{(i)}) \|_{2}^{2}
- 2 \frac{\eta_{t}}{d}\big< \mathbf{w}_t - \mathbf{w}_{t}^{(i)} , \nabla F_{S}(\mathbf{w}_{t}) - \nabla F_{S_i}(\mathbf{w}_{t}^{(i)}) \big>.
\end{align}
Due to the definitions of $F_S,F_{S_i}, F_{S_{-i}}$, we have the following inequality
\begin{align*}
&\| \nabla F_{S}(\mathbf{w}_{t}) - \nabla F_{S_i}(\mathbf{w}_{t}^{(i)}) \|_{2}^{2}
= \big\| \nabla F_{S_{-i}}(\mathbf{w}_{t})
- \nabla F_{S_{-i}}(\mathbf{w}_{t}^{(i)})
+ \frac{1}{n(n-1)}\big(\sum\limits_{k \neq i}\big[ \nabla f(\mathbf{w}_{t};z_{k},z_{i})  - \nabla f(\mathbf{w}_{t}^{(i)};z_{k},z_{i}') \big] \\
& + \sum\limits_{l \neq i}\big[\nabla f(\mathbf{w}_{t};z_{i},z_{l}) - \nabla f(\mathbf{w}_{t}^{(i)};z_{i}',z_{l}) \big]\big) \big\|_{2}^{2}
\leq 2 \| \nabla F_{S_{-i}}(\mathbf{w}_{t})
- \nabla F_{S_{-i}}(\mathbf{w}_{t}^{(i)}) \|_{2}^{2}\\
& + \frac{2}{n^2(n-1)^2} \big\| \sum\limits_{k \neq i}\big[ \nabla f(\mathbf{w}_{t};z_{k},z_{i}) - \nabla f(\mathbf{w}_{t}^{(i)};z_{k},z_{i}') \big]
+ \sum\limits_{l \neq i}\big[\nabla f(\mathbf{w}_{t};z_{i},z_{l}) - \nabla f(\mathbf{w}_{t}^{(i)};z_{i}',z_{l}) \big] \big\|_{2}^{2}.
\end{align*}
Considering the coercivity of $F_{S_{-i}}$, we apply Lemma \ref{le3} to give
\begin{align*}
&\Big< \mathbf{w}_t - \mathbf{w}_{t}^{(i)} , \nabla F_{S}(\mathbf{w}_{t}) - \nabla F_{S_i}(\mathbf{w}_{t}^{(i)}) \Big>
= \Big< \mathbf{w}_t - \mathbf{w}_{t}^{(i)} , \nabla F_{S_{-i}}(\mathbf{w}_{t}) - \nabla F_{S_{-i}}(\mathbf{w}_{t}^{(i)}) \Big>
+ \frac{1}{n(n-1)}\Big< \mathbf{w}_t - \mathbf{w}_{t}^{(i)} , \notag\\
&  \sum\limits_{k \neq i}\big[ \nabla f(\mathbf{w}_{t};z_{k},z_{i}) - \nabla f(\mathbf{w}_{t}^{(i)};z_{k},z_{i}') \big] + \sum\limits_{l \neq i}\big[\nabla f(\mathbf{w}_{t};z_{i},z_{l}) - \nabla f(\mathbf{w}_{t}^{(i)};z_{i}',z_{l}) \big] \Big>
\geq \frac{1}{L}\| \nabla F_{S_{-i}}(\mathbf{w}_{t}) - \nabla F_{S_{-i}}(\mathbf{w}_{t}^{(i)}) \|_{2}^{2} \notag\\
&- \frac{1}{n(n-1)} \| \mathbf{w}_t - \mathbf{w}_{t}^{(i)} \|_{2}
\| \sum\limits_{k \neq i}\big[ \nabla f(\mathbf{w}_{t};z_{k},z_{i}) - \nabla f(\mathbf{w}_{t}^{(i)};z_{k},z_{i}') \big] + \sum\limits_{l \neq i}\big[\nabla f(\mathbf{w}_{t};z_{i},z_{l}) - \nabla f(\mathbf{w}_{t}^{(i)};z_{i}',z_{l}) \big] \|_{2}.
\end{align*}
Plugging the above two inequalities back into equation {(\ref{eq13})}, we further get
\begin{align*}
&\mathbb{E}_{A}\left[\|\mathbf{w}_{t+1} - \mathbf{w}_{t+1}^{(i)}\|_{2}^{2}\right]
\leq \mathbb{E}_{A}\left[\|\mathbf{w}_{t} - \mathbf{w}_{t}^{(i)}\|_{2}^{2}\right]
+ (\frac{2\eta_{t}^{2}}{d}-\frac{2\eta_{t}}{dL})\mathbb{E}_{A}\left[\| \nabla F_{S_{-i}}(\mathbf{w}_{t}) - \nabla F_{S_{-i}}(\mathbf{w}_{t}^{(i)}) \|_{2}^{2}\right] \\
& + \frac{2\eta_{t}^{2}}{n^2(n-1)^2d}\mathbb{E}_{A}\Big[\| \sum\limits_{k \neq i}\big[ \nabla f(\mathbf{w}_{t};z_{k},z_{i}) - \nabla f(\mathbf{w}_{t}^{(i)};z_{k},z_{i}') \big]
+ \sum\limits_{l \neq i}\big[\nabla f(\mathbf{w}_{t};z_{i},z_{l}) - \nabla f(\mathbf{w}_{t}^{(i)};z_{i}',z_{l}) \big] \|_{2}^{2}\Big] \\
& + \frac{2\eta_t}{n(n-1)d} \mathbb{E}_{A}\Big[\|\mathbf{w}_t - \mathbf{w}_{t}^{(i)} \|_{2}
\| \sum\limits_{k \neq i}\big[ \nabla f(\mathbf{w}_{t};z_{k},z_{i}) - \nabla f(\mathbf{w}_{t}^{(i)};z_{k},z_{i}') \big] + \sum\limits_{l \neq i}\big[\nabla f(\mathbf{w}_{t};z_{i},z_{l}) - \nabla f(\mathbf{w}_{t}^{(i)};z_{i}',z_{l}) \big] \|_{2}\Big].
\end{align*}
With the definition $\mathfrak{C}_{t,i} = \| \sum_{k \neq i}\big[ \nabla f(\mathbf{w}_{t};z_{k},z_{i}) - \nabla f(\mathbf{w}_{t}^{(i)};z_{k},z_{i}') \big] + \sum_{l \neq i}\big[\nabla f(\mathbf{w}_{t};z_{i},z_{l}) - \nabla f(\mathbf{w}_{t}^{(i)};z_{i}',z_{l}) \big] \|_2$ for any $t \in [T]$ and $i \in [n]$, we note $\eta_t \leq 1/L$ and simplify the above inequality as
\begin{align*}
\mathbb{E}_{A}\left[\|\mathbf{w}_{t+1} - \mathbf{w}_{t+1}^{(i)}\|_{2}^{2}\right]
\leq \mathbb{E}_{A}\left[\|\mathbf{w}_{t} - \mathbf{w}_{t}^{(i)}\|_{2}^{2}\right]
+ \frac{2\eta_{t}^{2}}{n^2(n-1)^2d}\mathbb{E}_{A}[\mathfrak{C}_{t,i}^2]
+ \frac{2\eta_t}{n(n-1)d}\mathbb{E}_{A}[\| \mathbf{w}_t - \mathbf{w}_{t}^{(i)} \|_{2}\mathfrak{C}_{t,i}].
\end{align*}
This can be used recursively to give
\begin{align*}
&\mathbb{E}_{A}\left[\|\mathbf{w}_{t+1} - \mathbf{w}_{t+1}^{(i)}\|_{2}^{2}\right]
\leq \sum\limits_{j=1}^{t} \frac{2\eta_j}{n(n-1)d}(\mathbb{E}_{A}[\| \mathbf{w}_j - \mathbf{w}_{j}^{(i)} \|_{2}^{2}])^{\frac{1}{2}} (\mathbb{E}_A[\mathfrak{C}_{j,i}^2])^{\frac{1}{2}}
+ \sum\limits_{j=1}^{t} \frac{2\eta_{j}^{2}}{n^2(n-1)^2d}\mathbb{E}_{A}[\mathfrak{C}_{j,i}^2],
\end{align*}
where we have used the standard inequality $\mathbb{E}[XY] \leq (\mathbb{E}[X^2])^{\frac{1}{2}}(\mathbb{E}[Y^2])^{\frac{1}{2}}$.
Introducing $\widetilde{\Delta}_{t,i} = \max_{k \leq t}(\mathbb{E}_{A}[\| \mathbf{w}_k - \mathbf{w}_{k}^{(i)} \|_{2}^{2}])^{\frac{1}{2}}$ for any $t$ and $i$, we can view the above inequality as the following quadratic inequality
\begin{align*}
\mathbb{E}_{A}\left[\|\mathbf{w}_{t+1} - \mathbf{w}_{t+1}^{(i)}\|_{2}^{2}\right]
\leq \widetilde{\Delta}_{t,i}\sum\limits_{j=1}^{t} \frac{2\eta_j}{n(n-1)d} (\mathbb{E}_A[\mathfrak{C}_{j,i}^2])^{\frac{1}{2}}
+ \sum\limits_{j=1}^{t} \frac{2\eta_{j}^{2}}{n^2(n-1)^2d}\mathbb{E}_{A}[\mathfrak{C}_{j,i}^2].
\end{align*}
Since the right hand of the above inequality is an increasing function w.r.t. $t$, we further have
\begin{align*}
\widetilde{\Delta}_{t,i}^2
\leq \widetilde{\Delta}_{t,i}\sum\limits_{j=1}^{t} \frac{2\eta_j}{n(n-1)d}  (\mathbb{E}_A[\mathfrak{C}_{j,i}^2])^{\frac{1}{2}}
+ \sum\limits_{j=1}^{t} \frac{2\eta_{j}^{2}}{n^2(n-1)^2d}\mathbb{E}_{A}[\mathfrak{C}_{j,i}^2].
\end{align*}
Let $a,b > 0$.
If $x^2 \leq ax+b$, then we know $x \leq \sqrt{b+a^2/4} + a/2$, which can be used to give $x^2 \leq 2(b+a^2/4) + 2 \times a^2/4 = a^2 + 2b$.
This is useful for solving the above quadratic inequality of $\widetilde{\Delta}_{t,i}$ and implies the following inequality
\begin{align*}
\widetilde{\Delta}_{t,i}^2
\leq \left(\sum\limits_{j=1}^{t} \frac{2\eta_j}{n(n-1)d} (\mathbb{E}_A[\mathfrak{C}_{j,i}^2])^{\frac{1}{2}}\right)^2
+ \sum\limits_{j=1}^{t} \frac{4\eta_{j}^{2}}{n^2(n-1)^2d}\mathbb{E}_{A}[\mathfrak{C}_{j,i}^2]
\leq t\sum\limits_{j=1}^{t}\frac{4\eta_{j}^{2}}{n^2(n-1)^2d^2}
\mathbb{E}_{A}[\mathfrak{C}_{j,i}^2]
+ \sum\limits_{j=1}^{t} \frac{4\eta_{j}^{2}}{n^2(n-1)^2d}\mathbb{E}_{A}[\mathfrak{C}_{j,i}^2],
\end{align*}
where we have used the Cauchy inequality $(\sum_{j=1}^{t}a_k)^2 \leq t\sum_{j=1}^{t}a_{j}^{2}$ in the last step.
Then we take an average over $i \in  [n]$ in the above inequality and get the following result
\begin{align*}
\frac{1}{n}\sum\limits_{i=1}^{n}\widetilde{\Delta}_{t,i}^2
\leq \frac{4}{n^3(n-1)^2d}(\frac{t}{d}+1) \sum\limits_{j=1}^{t}\sum\limits_{i=1}^{n} \eta_{j}^{2}\mathbb{E}_{A}[\mathfrak{C}_{j,i}^2].
\end{align*}
Taking expectations w.r.t. $S,S'$ in both sides of this inequality further gives
\begin{align*}
&\frac{1}{n}\sum\limits_{i=1}^{n}\mathbb{E}_{S,S',A}\left[\|\mathbf{w}_{t+1} - \mathbf{w}_{t+1}^{(i)}\|_{2}^{2}\right] \\
&\leq \frac{32L}{n^3(n-1)d}(\frac{t}{d}+1) \sum\limits_{j=1}^{t}\sum\limits_{i=1}^{n} \eta_{j}^{2}\mathbb{E}_{S,S',A}\Big[ \sum_{k \neq i} \big(f(\mathbf{w}_{j};z_{k},z_{i}) + f(\mathbf{w}_{j}^{(i)};z_{k},z_{i}')\big)
+ \sum_{l \neq i}\big(f(\mathbf{w}_{j};z_{i},z_{l}) + f(\mathbf{w}_{j}^{(i)};z_{i}',z_{l})\big) \Big]\\
&= \frac{128L}{n^2d}(\frac{t}{d}+1)\sum\limits_{j=1}^{t}
\eta_{j}^{2}\mathbb{E}_{S,S',A}[F_S(\mathbf{w}_j)],
\end{align*}
where we have used the following inequality (by the self-bounding property of $L$-smooth function)
\begin{align*}
\mathbb{E}_{S,S',A}[\mathfrak{C}_{j,i}^2]
&= \mathbb{E}_{S,S',A}\Big[\| \sum_{k \neq i}\big[ \nabla f(\mathbf{w}_{j};z_{k},z_{i}) - \nabla f(\mathbf{w}_{j}^{(i)};z_{k},z_{i}') \big] + \sum_{l \neq i}\big[\nabla f(\mathbf{w}_{j};z_{i},z_{l}) - \nabla f(\mathbf{w}_{j}^{(i)};z_{i}',z_{l}) \big] \|_{2}^{2}\Big] \\
&\leq 4(n-1)\mathbb{E}_{S,S',A}\Big[\sum_{k \neq i}\| \nabla f(\mathbf{w}_{j};z_{k},z_{i})\|_{2}^{2} + \sum_{k \neq i}\| \nabla f(\mathbf{w}_{j}^{(i)};z_{k},z_{i}')\|_{2}^{2}
+ \sum_{l \neq i}\| \nabla f(\mathbf{w}_{j};z_{i},z_{l})\|_{2}^{2} + \sum_{l \neq i}\| \nabla f(\mathbf{w}_{j}^{(i)};z_{i}',z_{l}) \|_{2}^{2}\Big] \\
&\leq 8(n-1)L\mathbb{E}_{S,S',A}\Big[ \sum_{k \neq i} \big(f(\mathbf{w}_{j};z_{k},z_{i}) + f(\mathbf{w}_{j}^{(i)};z_{k},z_{i}')\big)
+ \sum_{l \neq i}\big(f(\mathbf{w}_{j};z_{i},z_{l}) + f(\mathbf{w}_{j}^{(i)};z_{i}',z_{l})\big) \Big]
\end{align*}
and the following equations (by the symmetry between $z_i$ and $z_i'$)
\begin{align*}
\mathbb{E}_{S,S',A} \left[\sum\limits_{l \neq i} f(\mathbf{w}_{j}^{(i)};z_{i}',z_{l})\right]
= \mathbb{E}_{S,S',A} \left[\sum\limits_{l \neq i} f(\mathbf{w}_{j};z_{i},z_{l}) \right]
\quad {\rm and} \quad
\mathbb{E}_{S,S',A} \left[\sum\limits_{k \neq i} f(\mathbf{w}_{j}^{(i)};z_{k},z_{i}')\right]
= \mathbb{E}_{S,S',A} \left[\sum\limits_{k \neq i} f(\mathbf{w}_{j};z_{k},z_{i})\right].
\end{align*}
This completes the proof for the $\ell_2$ on-average argument stability bound (\ref{eq9}).
Then we turn to the excess risk and plug the stability bound back into part (b) of Lemma \ref{le1} to get
\begin{align*}
&\mathbb{E}_{S,A}\left[F(\mathbf{w}_{t+1}) - F_{S}(\mathbf{w}_{t+1})\right]
\leq\frac{256L(L+\gamma)}{n^2d}(\frac{t}{d}+1)\sum\limits_{j=1}^{t}
\eta_{j}^{2}\mathbb{E}_{S,A}[F_S(\mathbf{w}_j)] +  \frac{L}{\gamma} \mathbb{E}_{S,A} \left[ F_{S}(\mathbf{w}_{t+1})\right].
\end{align*}
For the above inequality, we sum both sides by $\mathbb{E}_{S,A}\left[F_S(\mathbf{w}_{t+1}) - F_{S}(\mathbf{w})\right]$ and use the decomposition $F_S(\mathbf{w}_{j}) = F_S(\mathbf{w}_{j}) - F_S(\mathbf{w}) + F_S(\mathbf{w})$ to derive as below
\begin{align*}
\mathbb{E}_{S,A}\left[F(\mathbf{w}_{t+1}) - F_{S}(\mathbf{w})\right]
&\leq \mathbb{E}_{S,A}\big[F_S(\mathbf{w}_{t+1}) - F_{S}(\mathbf{w})\big] +\frac{L}{\gamma}\mathbb{E}_{S,A}\big[F_{S}(\mathbf{w}_{t+1})- F_S(\mathbf{w}) + F_S(\mathbf{w}) \big] \\
&+\frac{256L(L+\gamma)}{n^2d}(\frac{t}{d}+1)\sum\limits_{j=1}^{t}
\eta_{j}^{2}\mathbb{E}_{S,A}\left[F_{S}(\mathbf{w}_{j}) - F_S(\mathbf{w}) + F_S(\mathbf{w})\right].
\end{align*}
Applying (\ref{eq4}) and (\ref{eq5}) in the above inequality gives
\begin{align*}
&\mathbb{E}_{S,A}\left[F(\mathbf{w}_{t+1}) - F(\mathbf{w}^{*})\right]
\leq\left[\frac{L}{\gamma}+\frac{256L(L+\gamma)}{n^2d}(\frac{t}{d}+1)\sum\limits_{j=1}^{t}
\eta_{j}^{2}\right]\times F(\mathbf{w}^{*})\\
&+\frac{d(1+L\gamma^{-1})}{2\sum_{j=1}^{t+1}\eta_j}
\left(\|\mathbf{w}_{1}-\mathbf{w}^{*}\|_{2}^{2}
+ 2\eta_{1}F_{S}(\mathbf{w}_{1}) \right)
+\frac{128L(L+\gamma)(t+d)}{n^2d}\left(\eta_1\| \mathbf{w}_1 - \mathbf{w}^{*} \|_{2}^{2} + 2  \eta_{1}^{2} F_S(\mathbf{w}_1)\right),
\end{align*}
where we have set $\mathbf{w} = \mathbf{w}^{*}$ to prove the excess risk bound (\ref{eq10}).
When step sizes are fixed, the excess risk bound with $T$ iterations becomes
\begin{align*}
\mathbb{E}_{S,A}\left[F(\mathbf{w}_{T}) - F(\mathbf{w}^{*})\right]
= O\left(\frac{d(1+L\gamma^{-1})}{T}
+\frac{L(L+\gamma)(T+d)} {n^2d}\right)
+ O\left(\frac{L}{\gamma}+\frac{LT(L+\gamma)(T+d)}{n^2d^2}
 \right) \times F(\mathbf{w}^{*}),
\end{align*}
To prove (\ref{eq11}), we choose $\gamma = \frac{LT}{d}$ and have
$$
\mathbb{E}_{S,A}\left[F(\mathbf{w}_{T}) - F(\mathbf{w}^{*})\right]
=O\Big(\frac{d}{T} + \frac{L^2 T^3}{n^2 d^3} \Big).
$$
Then we choose $T \asymp n^{\frac{1}{2}}dL^{-\frac{1}{2}}$ to get
$$
\mathbb{E}_{S,A}\left[F(\mathbf{w}_{T}) - F(\mathbf{w}^{*})\right]
=O\Big(\sqrt{\frac{L}{n}} \Big).
$$
Turning to the other convergence rate, if $F(\mathbf{w}^{*}) = O(d/T)$, then we choose $\gamma = \frac{nd}{T}$ and have
$$
\mathbb{E}_{S,A}\left[F(\mathbf{w}_{T}) - F(\mathbf{w}^{*})\right]
=O\left(\frac{L}{n} + \frac{d}{T} + \frac{L^2T}{n^2d} \right).
$$
Finally, we choose $T \asymp ndL^{-1}$ to derive the convergence rate (\ref{eq12})
$$
\mathbb{E}_{S,A}\left[F(\mathbf{w}_{T}) - F(\mathbf{w}^{*})\right]
= O\left(\frac{L}{n} \right).
$$
The proof is completed.
\end{proof}

\section{D \quad Proof for Strongly Convex Case}

\begin{theorem}\label{th4}
{\rm \textbf{(restated)}}.
Let Assumptions \ref{as2},\ref{as3},\ref{as5} hold.
Let $\{\mathbf{w}_{t}\}$, $\{\mathbf{w}_{t}^{(i)}\}$ be produced by (\ref{eq1}) with $\eta_t \leq \beta/L$ for any $\beta \in (0,1)$ based on $S$ and $S_i$, respectively.
Then the $\ell_2$ on-average  argument stability is shown as below
\begin{align} \label{eq14}
&\frac{1}{n}\sum\limits_{i=1}^{n}\mathbb{E}_{S,S',A}\left[\|\mathbf{w}_{t+1} - \mathbf{w}_{t+1}^{(i)}\|_{2}^{2}\right] \notag\\
&\leq \frac{128L}{n^2d}\sum\limits_{j=1}^{t}
\Big(\frac{t}{d}\prod\limits_{k=j+1}^{t}
\big(1-\frac{2\eta_k(1-\beta)(n-2)\sigma}{nd}\big)^2+
\prod\limits_{k=j+1}^{t}
\big(1-\frac{2\eta_k(1-\beta)(n-2)\sigma}{nd}\big)\Big)
\eta_{j}^{2}\mathbb{E}_{S,A}[F_S(\mathbf{w}_j)].
\end{align}
Let step sizes be fixed as $\eta \leq 1 / \widetilde{L}$. For any $\gamma > 0$, we develop the excess risk bound as
\begin{align}\label{eq15}
\mathbb{E}_{S,A}\left[F(\mathbf{w}_{T+1}) - F_S(\mathbf{w}_S)\right] &= O\left(\big(1+\frac{L}{\gamma}\big)(1-\eta\sigma/d)^T + \frac{L(d+T)(L+\gamma)}{(n-2)^2\sigma^2(1-\beta)^2}\right) \notag\\
&+ O\left(\frac{L}{\gamma} + \frac{L(d+T)(L+\gamma)}{(n-2)^2\sigma^2(1-\beta)^2}\right)\times \mathbb{E} [ F_S(\mathbf{w}_S) ]
\end{align}
In particular, we can choose $T \asymp d\sigma^{-1}\log(n\sigma L^{-1})$ in the above convergence rate and get
\begin{align}\label{eq16}
\mathbb{E}_{S,A}\left[F(\mathbf{w}_{T+1}) - F_S(\mathbf{w}_S)\right]
= O\left(\frac{Ld^{\frac{1}{2}}}{n\sigma^{\frac{3}{2}}}
\sqrt{\log\frac{n\sigma}{L}}\right).
\end{align}
\end{theorem}

\begin{proof}
Referring to the proof under convex case, we have the following equation
\begin{align*}
\mathbb{E}_{i_t}[\|\mathbf{w}_{t+1} - \mathbf{w}_{t+1}^{(i)}\|_{2}^{2}]
= \| \mathbf{w}_t - \mathbf{w}_{t}^{(i)} \|_{2}^{2}
+ \frac{\eta_{t}^{2}}{d}\| \nabla F_{S}(\mathbf{w}_{t}) - \nabla F_{S_i}(\mathbf{w}_{t}^{(i)}) \|_{2}^{2}
- 2 \frac{\eta_{t}}{d}\big< \mathbf{w}_t - \mathbf{w}_{t}^{(i)} , \nabla F_{S}(\mathbf{w}_{t}) - \nabla F_{S_i}(\mathbf{w}_{t}^{(i)}) \big>
\end{align*}
and the following inequality
\begin{align*}
&\| \nabla F_{S}(\mathbf{w}_{t}) - \nabla F_{S_i}(\mathbf{w}_{t}^{(i)}) \|_{2}^{2}
\leq 2 \| \nabla F_{S_{-i}}(\mathbf{w}_{t})
- \nabla F_{S_{-i}}(\mathbf{w}_{t}^{(i)}) \|_{2}^{2}\\
&+ \frac{2}{n^2(n-1)^2} \big\| \sum\limits_{k \neq i}\big[ \nabla f(\mathbf{w}_{t};z_{k},z_{i}) - \nabla f(\mathbf{w}_{t}^{(i)};z_{k},z_{i}') \big]
+ \sum\limits_{l \neq i}\big[\nabla f(\mathbf{w}_{t};z_{i},z_{l}) - \nabla f(\mathbf{w}_{t}^{(i)};z_{i}',z_{l}) \big] \big\|_{2}^{2}.
\end{align*}
Resembling Lemma \ref{le3} for convex empirical risks, we present the coercivity property as the following inequality
\begin{align*}
&\Big< \mathbf{w}_t - \mathbf{w}_{t}^{(i)} , \nabla F_{S}(\mathbf{w}_{t}) - \nabla F_{S_i}(\mathbf{w}_{t}^{(i)}) \Big> \\
&\geq \frac{\beta}{L}\| \nabla F_{S_{-i}}(\mathbf{w}_{t}) - \nabla F_{S_{-i}}(\mathbf{w}_{t}^{(i)}) \|_{2}^{2}
+ (1-\beta)\sigma'\| \mathbf{w}_t - \mathbf{w}_{t}^{(i)} \|_{2}^{2}
- \frac{1}{n}\| \mathbf{w}_t - \mathbf{w}_{t}^{(i)} \|_{2}\| \nabla f(\mathbf{w}_{t};z_{i}) - \nabla f(\mathbf{w}_{t};z_{i}') \|_{2},
\end{align*}
where the strong convexity parameter is $\sigma' = \frac{n-2}{n} \sigma$ since $F_{S_{-i}}(\mathbf{w}) = \frac{1}{n(n-1)}\sum_{j,j'\in [n] / i : j \neq j'}f(\mathbf{w};z_j,z_{j'})$ is $\sigma'$-strongly convex.
The above three formulas jointly imply the inequality as below
\begin{align*}
&\mathbb{E}_{A}\left[\|\mathbf{w}_{t+1} - \mathbf{w}_{t+1}^{(i)}\|_{2}^{2}\right]
\leq \Big(1-\frac{2\eta_t(1-\beta)\sigma'}{d}\Big)
\mathbb{E}_{A}\left[\|\mathbf{w}_{t} - \mathbf{w}_{t}^{(i)}\|_{2}^{2}\right]
+ (\frac{2\eta_{t}^{2}}{d}-\frac{2\eta_{t}\beta}{dL})
\mathbb{E}_{A}\left[\| \nabla F_{S_{-i}}(\mathbf{w}_{t}) - \nabla F_{S_{-i}}(\mathbf{w}_{t}^{(i)}) \|_{2}^{2}\right] \\
& + \frac{2\eta_{t}^{2}}{n^2(n-1)^2d}\mathbb{E}_{A}\left[\| \sum\limits_{k \neq i}\big[ \nabla f(\mathbf{w}_{t};z_{k},z_{i}) - \nabla f(\mathbf{w}_{t}^{(i)};z_{k},z_{i}') \big]
+ \sum\limits_{l \neq i}\big[\nabla f(\mathbf{w}_{t};z_{i},z_{l}) - \nabla f(\mathbf{w}_{t}^{(i)};z_{i}',z_{l}) \big] \|_{2}^{2}\right] \\
& + \frac{2\eta_t}{n(n-1)d}\| \mathbf{w}_t - \mathbf{w}_{t}^{(i)} \|_{2}
\| \sum\limits_{k \neq i}\big[ \nabla f(\mathbf{w}_{t};z_{k},z_{i}) - \nabla f(\mathbf{w}_{t}^{(i)};z_{k},z_{i}') \big] + \sum\limits_{l \neq i}\big[\nabla f(\mathbf{w}_{t};z_{i},z_{l}) - \nabla f(\mathbf{w}_{t}^{(i)};z_{i}',z_{l}) \big] \|_{2},
\end{align*}
from which we further get the following inequality with $\eta_t \leq \beta/L$ and $\mathfrak{C}_{t,i} = \| \sum_{k \neq i}\big[ \nabla f(\mathbf{w}_{t};z_{k},z_{i}) - \nabla f(\mathbf{w}_{t}^{(i)};z_{k},z_{i}') \big] + \sum_{l \neq i}\big[\nabla f(\mathbf{w}_{t};z_{i},z_{l}) - \nabla f(\mathbf{w}_{t}^{(i)};z_{i}',z_{l}) \big] \|_2$
\begin{align*}
\mathbb{E}_{A}\left[\|\mathbf{w}_{t+1} - \mathbf{w}_{t+1}^{(i)}\|_{2}^{2}\right]
\leq \Big(1-\frac{2\eta_t(1-\beta)\sigma'}{d}\Big)
\mathbb{E}_{A}\left[\|\mathbf{w}_{t} - \mathbf{w}_{t}^{(i)}\|_{2}^{2}\right]
+\frac{2\eta_{t}^{2}}{n^2(n-1)^2d}\mathbb{E}_{A}[\mathfrak{C}_{t,i}^2]
+ \frac{2\eta_t}{n(n-1)d}\mathbb{E}_{A}[\| \mathbf{w}_t - \mathbf{w}_{t}^{(i)} \|_{2}\mathfrak{C}_{t,i}].
\end{align*}
This inequality can be used recursively to give
\begin{align*}
\mathbb{E}_{A}\left[\|\mathbf{w}_{t+1} - \mathbf{w}_{t+1}^{(i)}\|_{2}^{2}\right]
&\leq \sum\limits_{j=1}^{t}\frac{2\eta_j}{n(n-1)d} \prod\limits_{k=j+1}^{t} \Big(1-\frac{2\eta_k(1-\beta)\sigma'}{d}\Big) (\mathbb{E}_{A}[\| \mathbf{w}_j - \mathbf{w}_{j}^{(i)} \|_{2}^{2}])^{\frac{1}{2}} (\mathbb{E}_A[\mathfrak{C}_{j,i}^2])^{\frac{1}{2}} \\
&+ \sum\limits_{j=1}^{t} \frac{2\eta_{j}^{2}}{n^2(n-1)^2d}\prod\limits_{k=j+1}^{t} \Big(1-\frac{2\eta_k(1-\beta)\sigma'}{d}\Big)
\mathbb{E}_{A}[\mathfrak{C}_{j,i}^2],
\end{align*}
where we have also handled $\mathbb{E}_{A}[\| \mathbf{w}_t - \mathbf{w}_{t}^{(i)} \|_{2}\mathfrak{C}_{t,i}]$ by the standard inequality $\mathbb{E}[XY] \leq (\mathbb{E}[X^2])^{\frac{1}{2}}(\mathbb{E}[Y^2])^{\frac{1}{2}}$.
Then we further derive the following inequality
\begin{align*}
\widetilde{\Delta}_{t,i}^{2}
&\leq \sum\limits_{j=1}^{t}\frac{2\eta_j}{n(n-1)d} \prod\limits_{k=j+1}^{t} \Big(1-\frac{2\eta_k(1-\beta)\sigma'}{d}\Big) (\mathbb{E}_A[\mathfrak{C}_{j,i}^2])^{\frac{1}{2}} \widetilde{\Delta}_{t,i} \notag\\
&+ \sum\limits_{j=1}^{t} \frac{2\eta_{j}^{2}}{n^2(n-1)^2d}\prod\limits_{k=j+1}^{t} \Big(1-\frac{2\eta_k(1-\beta)\sigma'}{d}\Big)
\mathbb{E}_{A}[\mathfrak{C}_{j,i}^2],
\end{align*}
where $\widetilde{\Delta}_{t,i}$ shares the definition of $\max_{1 \leq j \leq t}(\mathbb{E}_{A}[\| \mathbf{w}_j - \mathbf{w}_{j}^{(i)} \|_{2}^{2}])^{\frac{1}{2}}$ as that in the convex case.
Then we solve this quadratic inequality and give
\begin{align*}
\widetilde{\Delta}_{t,i}^2
&\leq \left(\sum\limits_{j=1}^{t}\frac{2\eta_j}{n(n-1)d} \prod\limits_{k=j+1}^{t} \Big(1-\frac{2\eta_k(1-\beta)\sigma'}{d}\Big) (\mathbb{E}_A[\mathfrak{C}_{j,i}^2])^{\frac{1}{2}} \right)^2
+ \sum\limits_{j=1}^{t} \frac{4\eta_{j}^{2}}{n^2(n-1)^2d}\prod\limits_{k=j+1}^{t} \Big(1-\frac{2\eta_k(1-\beta)\sigma'}{d}\Big)
\mathbb{E}_{A}[\mathfrak{C}_{j,i}^2] \notag\\
&\leq t\sum\limits_{j=1}^{t}\frac{4\eta_{j}^{2}}{n^2(n-1)^2d^2}
\prod\limits_{k=j+1}^{t}
\Big(1-\frac{2\eta_k(1-\beta)\sigma'}{d}\Big)^2
\mathbb{E}_A[\mathfrak{C}_{j,i}^2]
+ \sum\limits_{j=1}^{t} \frac{4\eta_{j}^{2}}{n^2(n-1)^2d}\prod\limits_{k=j+1}^{t} \Big(1-\frac{2\eta_k(1-\beta)\sigma'}{d}\Big)
\mathbb{E}_{A}[\mathfrak{C}_{j,i}^2] \notag\\
&=\sum\limits_{j=1}^{t} \frac{4\eta_{j}^{2}}{n^2(n-1)^2d}
\Big(\frac{t}{d}\prod\limits_{k=j+1}^{t}
\big(1-\frac{2\eta_k(1-\beta)\sigma'}{d}\big)^2+
\prod\limits_{k=j+1}^{t}\big(1-\frac{2\eta_k(1-\beta)\sigma'}{d}\big)\Big)
\mathbb{E}_{A}[\mathfrak{C}_{j,i}^2],
\end{align*}
where we should note that $\prod_{k=t+1}^{t} \Big(1-\frac{2\eta_k(1-\beta)\sigma'}{d}\Big) = \prod_{k=t+1}^{t} \Big(1-\frac{2\eta_k(1-\beta)\sigma'}{d}\Big)^2 = 1$.
Taking an average over $i \in  [n]$ in the above inequality, we further have
\begin{align*}
\frac{1}{n}\sum\limits_{i=1}^{n}\widetilde{\Delta}_{t,i}^2
\leq \sum\limits_{j=1}^{t}\sum\limits_{i=1}^{n} \frac{4\eta_{j}^{2}}{n^3(n-1)^2d}
\Big(\frac{t}{d}\prod\limits_{k=j+1}^{t}
\big(1-\frac{2\eta_k(1-\beta)\sigma'}{d}\big)^2+
\prod\limits_{k=j+1}^{t}\big(1-\frac{2\eta_k(1-\beta)\sigma'}{d}\big)\Big)
\mathbb{E}_{A}[\mathfrak{C}_{j,i}^2].
\end{align*}
Then we take expectations w.r.t. $S,S'$ in both sides of the above inequality to get
\begin{align} \label{eq17}
&\frac{1}{n}\sum\limits_{i=1}^{n}\mathbb{E}_{S,S',A}\left[\|\mathbf{w}_{t+1} - \mathbf{w}_{t+1}^{(i)}\|_{2}^{2}\right] \notag\\
&\leq \frac{128L}{n^2d}\sum\limits_{j=1}^{t}
\Big(\frac{t}{d}\prod\limits_{k=j+1}^{t}
\big(1-\frac{2\eta_k(1-\beta)\sigma'}{d}\big)^2+
\prod\limits_{k=j+1}^{t}
\big(1-\frac{2\eta_k(1-\beta)\sigma'}{d}\big)\Big)
\eta_{j}^{2}\mathbb{E}_{S,S',A}[F_S(\mathbf{w}_j)],
\end{align}
where we have used the following formulas
\begin{align*}
\mathbb{E}_{S,S',A}[\mathfrak{C}_{j,i}^2]
\leq 8(n-1)L\mathbb{E}_{S,S',A}\Big[ \sum_{k \neq i} \big(f(\mathbf{w}_{j};z_{k},z_{i}) + f(\mathbf{w}_{j}^{(i)};z_{k},z_{i}')\big)
+ \sum_{l \neq i}\big(f(\mathbf{w}_{j};z_{i},z_{l}) + f(\mathbf{w}_{j}^{(i)};z_{i}',z_{l})\big) \Big],
\end{align*}
\begin{align*}
\mathbb{E}_{S,S',A} \left[\sum\limits_{l \neq i} f(\mathbf{w}_{j}^{(i)};z_{i}',z_{l})\right]
= \mathbb{E}_{S,S',A} \left[\sum\limits_{l \neq i} f(\mathbf{w}_{j};z_{i},z_{l})\right]
\quad {\rm and} \quad
\mathbb{E}_{S,S',A} \left[\sum\limits_{k \neq i} f(\mathbf{w}_{j}^{(i)};z_{k},z_{i}')\right]
= \mathbb{E}_{S,S',A} \left[\sum\limits_{k \neq i} f(\mathbf{w}_{j};z_{k},z_{i})\right].
\end{align*}
Plugging $\sigma' = \frac{n-2}{n}\sigma$ into the above inequality yields the stability bound (\ref{eq14}).
Then we turn to the corresponding excess risk bound.
Noting the following two inequalities
\begin{align*}
&\sum\limits_{j=1}^{t}\Big(\frac{2\eta_j(1-\beta)\sigma'}{d}\Big)^2
\prod\limits_{k=j+1}^{t} \Big(1-\frac{2\eta_k(1-\beta)\sigma'}{d}\Big)^2
=\sum\limits_{j=1}^{t}\Big[1-\Big(1-\big(\frac{2\eta_j
(1-\beta)\sigma'}{d}\big)^2\Big)\Big]
\prod\limits_{k=j+1}^{t} \Big(1-\frac{2\eta_k(1-\beta)\sigma'}{d}\Big)^2 \\
&= \sum\limits_{j=1}^{t}\prod\limits_{k=j+1}^{t} \Big(1-\frac{2\eta_k(1-\beta)\sigma'}{d}\Big)^2 - \sum\limits_{j=1}^{t}\Big(1+\frac{2\eta_j(1-\beta)\sigma'}{d}\Big)
\Big(1-\frac{2\eta_j(1-\beta)\sigma'}{d}\Big)\prod\limits_{k=j+1}^{t} \Big(1-\frac{2\eta_k(1-\beta)\sigma'}{d}\Big)^2 \\
&\leq \sum\limits_{j=1}^{t}\prod\limits_{k=j+1}^{t} \Big(1-\frac{2\eta_k(1-\beta)\sigma'}{d}\Big)^2 - \sum\limits_{j=1}^{t}\Big(1+\frac{2\eta_j(1-\beta)\sigma'}{d}\Big)
\Big(1-\frac{2\eta_j(1-\beta)\sigma'}{d}\Big)^2\prod\limits_{k=j+1}^{t} \Big(1-\frac{2\eta_k(1-\beta)\sigma'}{d}\Big)^2 \\
&= \sum\limits_{j=1}^{t}\prod\limits_{k=j+1}^{t} \Big(1-\frac{2\eta_k(1-\beta)\sigma'}{d}\Big)^2
- \sum\limits_{j=1}^{t}
\prod\limits_{k=j}^{t} \Big(1-\frac{2\eta_k(1-\beta)\sigma'}{d}\Big)^2
- \sum\limits_{j=1}^{t} \frac{2\eta_j(1-\beta)\sigma'}{d}
\prod\limits_{k=j}^{t} \Big(1-\frac{2\eta_k(1-\beta)\sigma'}{d}\Big)^2\\
&= 1 - \prod\limits_{k=1}^{t} \Big(1-\frac{2\eta_k(1-\beta)\sigma'}{d}\Big)^2
- \sum\limits_{j=1}^{t} \frac{2\eta_j(1-\beta)\sigma'}{d} \prod\limits_{k=1}^{t} \Big(1-\frac{2\eta_k(1-\beta)\sigma'}{d}\Big)
\leq 1
\end{align*}
and
\begin{align*}
&\sum\limits_{j=1}^{t}\Big(\frac{2\eta_j(1-\beta)\sigma'}{d}\Big)^2
\prod\limits_{k=j+1}^{t} \Big(1-\frac{2\eta_k(1-\beta)\sigma'}{d}\Big)
=\sum\limits_{j=1}^{t}\Big[1-\Big(1-\big(\frac{2\eta_j
(1-\beta)\sigma'}{d}\big)^2\Big)\Big]
\prod\limits_{k=j+1}^{t} \Big(1-\frac{2\eta_k(1-\beta)\sigma'}{d}\Big) \\
&= \sum\limits_{j=1}^{t}\prod\limits_{k=j+1}^{t} \Big(1-\frac{2\eta_k(1-\beta)\sigma'}{d}\Big) - \sum\limits_{j=1}^{t}\Big(1+\frac{2\eta_j(1-\beta)\sigma'}{d}\Big)
\Big(1-\frac{2\eta_j(1-\beta)\sigma'}{d}\Big)\prod\limits_{k=j+1}^{t} \Big(1-\frac{2\eta_k(1-\beta)\sigma'}{d}\Big) \\
&= \sum\limits_{j=1}^{t}\prod\limits_{k=j+1}^{t} \Big(1-\frac{2\eta_k(1-\beta)\sigma'}{d}\Big)
- \sum\limits_{j=1}^{t}
\prod\limits_{k=j}^{t} \Big(1-\frac{2\eta_k(1-\beta)\sigma'}{d}\Big)
- \sum\limits_{j=1}^{t} \frac{2\eta_j(1-\beta)\sigma'}{d}
\prod\limits_{k=j}^{t} \Big(1-\frac{2\eta_k(1-\beta)\sigma'}{d}\Big)\\
&= 1 - \prod\limits_{k=1}^{t} \Big(1-\frac{2\eta_k(1-\beta)\sigma'}{d}\Big)
- \sum\limits_{j=1}^{t} \frac{2\eta_j(1-\beta)\sigma'}{d} \prod\limits_{k=1}^{t} \Big(1-\frac{2\eta_k(1-\beta)\sigma'}{d}\Big)
\leq 1,
\end{align*}
we apply them to simplify the inequality (\ref{eq17}) as
\begin{align*}
\frac{1}{n}\sum\limits_{i=1}^{n}\mathbb{E}_{S,S',A}\left[\|\mathbf{w}_{t+1} - \mathbf{w}_{t+1}^{(i)}\|_{2}^{2}\right]
&\leq \frac{128L}{n^2d} \times
\frac{d^2}{4(1-\beta)^2\sigma'^2} \times
( 1 + \frac{t}{d} )
\mathbb{E}_{S,S',A}\max_{1 \leq j \leq t}[F_S(\mathbf{w}_j)] \\
&= \frac{32L(d+t)}{n^2(1-\beta)^2\sigma'^2}
\max_{1 \leq j \leq t}\mathbb{E}_{S,A}[F_S(\mathbf{w}_j)].
\end{align*}
Then we plug the above stability bound into part (b) of Lemma \ref{le1} and have
\begin{align*}
&\mathbb{E}_{S,A}\left[F(\mathbf{w}_{t+1}) - F_{S}(\mathbf{w}_{t+1})\right]
\leq \frac{64L(d+t)(L+\gamma)}{n^2\sigma'^2(1-\beta)^2}
\max_{1 \leq j \leq t}\mathbb{E}_{S,A}[F_S(\mathbf{w}_j)]
+  \frac{L}{\gamma} \mathbb{E}_{S,A} \left[ F_{S}(\mathbf{w}_{t+1})\right].
\end{align*}
For the above inequality, we sum both sides by $\mathbb{E}_{S,A}\left[F_S(\mathbf{w}_{t+1}) - F_{S}(\mathbf{w}_S)\right]$ and use the decomposition $F_S(\mathbf{w}_{j}) = F_S(\mathbf{w}_{j}) - F_S(\mathbf{w}_S) + F_S(\mathbf{w}_S)$ to derive as below
\begin{align*}
\mathbb{E}_{S,A}\left[F(\mathbf{w}_{t+1}) - F_{S}(\mathbf{w}_S)\right]
&\leq \mathbb{E}_{S,A}\big[F_S(\mathbf{w}_{t+1}) - F_{S}(\mathbf{w}_S)\big] +\frac{L}{\gamma}\mathbb{E}_{S,A}\big[F_{S}(\mathbf{w}_{t+1})- F_S(\mathbf{w}_S) + F_S(\mathbf{w}_S) \big] \\
&+ \frac{64L(d+t)(L+\gamma)}{n^2\sigma'^2(1-\beta)^2}\max_{1 \leq j \leq t}\mathbb{E}_{S,A}\left[F_{S}(\mathbf{w}_{j}) - F_S(\mathbf{w}_S) + F_S(\mathbf{w}_S)\right].
\end{align*}
With the optimization error bound (\ref{eq6}) and the above inequality, we give
\begin{align*}
&\mathbb{E}_{S,A}\left[F(\mathbf{w}_{t+1}) - F_S(\mathbf{w}_S)\right]
\leq \big(1+\frac{L}{\gamma}\big)(1-\eta_{t}\sigma/d)^{t}
\mathbb{E} [F_{S}(\mathbf{w}_{1})- F_{S}(\mathbf{w}_{S})]\\
&+ \big( \frac{L}{\gamma} + \frac{64L(d+t)(L+\gamma)}{n^2\sigma'^2(1-\beta)^2} \big)\times \mathbb{E} [F_S(\mathbf{w}_S)]
+ \frac{64L(d+t)(L+\gamma)}{n^2\sigma'^2(1-\beta)^2}
\mathbb{E}[F_{S}(\mathbf{w}_{1})- F_{S}(\mathbf{w}_{S})].
\end{align*}
Setting the number of iterations as $T$ and noting $\sigma' = \frac{n-2}{n}\sigma$ in the above inequality can prove the convergence rate (\ref{eq15}).
Turning to the convergence rate (\ref{eq16}), we let $\gamma \asymp n\sigma/\sqrt{T}$ and have
$$
\mathbb{E}_{S,A}\left[F(\mathbf{w}_{T+1}) - F_S(\mathbf{w}_S)\right]
= O\left(\frac{L\sqrt{T}}{n\sigma}+(1-\frac{\eta\sigma}{d})^T\right),
$$
where we choose $T \asymp d\sigma^{-1}\log(n\sigma L^{-1})$ to give
$$
\mathbb{E}_{S,A}\left[F(\mathbf{w}_{T+1}) - F_S(\mathbf{w}_S)\right]
= O\left(\frac{Ld^{\frac{1}{2}}}{n\sigma^{\frac{3}{2}}}
\sqrt{\log\frac{n\sigma}{L}}\right).
$$
This completes the proof for the strongly convex case.
\end{proof}

\section{E \quad Experimental Results}

In this section, we choose the example of AUC maximization to verify the theoretical results on stability measures.
As shown below, the restatement presents the paradigm of AUC maximization tasks, where the relative loss functions play a key role in the development of experiments.
Besides, since both the loss functions are convex, all the following discussions for experimental results are developed based on the theoretical results in convex case.

{\textbf{AUC Maximization. (restated)}.}
AUC score is applied to measure the performance of classification models for imbalanced data.
With the binary output space $\mathcal{Y} = \left\{+1,-1\right\}$, it shows the probability that the model $h_{\mathbf{w}} : \mathcal{X} \mapsto \mathbb{R}$ scores a positive instance higher than a negative instance.
Therefore, the loss function for usually takes the form of $f(\mathbf{w};z,z') = g(\mathbf{w}^\top (x-x'))\mathbb{I}_{[y=1,y'=-1]}$,
where we choose the logistic loss $\phi(t) = \log(1+exp(-t))$ or the hinge loss $\phi(t) = \max\left\{1-t,0\right\}$.

With the task of AUC maximization, we apply RCD to show the stability results for pairwise learning.
We consider the following datasets from LIBSVM \cite{c:2011} and measure the stability of RCD on these datasets.
We follow the experimental settings of SGD for pairwise learning \cite{l:2021} and compare the results of RCD and SGD.
In each experiment, we randomly choose $80$ percents of each dataset as the training set $S$.
Then we perturb a a signal example of $S$ to construct the neighboring dataset $S'$.
We apply RCD or SGD to $S,S'$ and get two iterate sequences, with which we plot the Euclidean distance $\Delta_t =  \| \mathbf{w}_t - \mathbf{w}_t'\|_2$ for each iteration.
While the learning rates are set as $\eta_t = \eta/\sqrt{T}$ with $\eta \in \left\{ 0.05, 0.25, 1, 4 \right\}$ for RCD, we only compare RCD and SGD under the setting of $\eta = 0.05$.
Letting $n$ be the sample size, we report $\Delta_t$ as a function of $T/n$ (the number of passes).
We repeat the experiments $100$ times, and consider the average and the standard deviation.
Results for the hinge loss and the logistic loss are shown in the following, respectively.

\begin{table}[H]
\captionsetup{labelformat=empty}
\captionsetup{aboveskip=15pt}
\centering
\caption{Table E.1: Description of the datasets.}
\begin{tabular}{|c|c|c|c|c|c|c|c|}
\hline
dataset & inst \, feat &  dataset & inst \, feat &  dataset & inst \, feat &  dataset & inst \, feat \\
\hline
a3a & 3185 122 & gisette & 6000 5000 & madelon & 2000 500 & usps & 7291 256 \\
\hline
\end{tabular}
\end{table}

\begin{figure}[H]
    \captionsetup{labelformat=empty}
    \captionsetup{aboveskip=17pt}
    \centering
    \begin{minipage}{0.24\textwidth}
        \includegraphics[width=\linewidth]{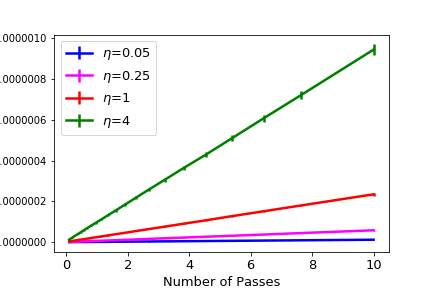}
        \caption{(a) RCD for a3a}
    \end{minipage}\hfill
    \begin{minipage}{0.24\textwidth}
        \includegraphics[width=\linewidth]{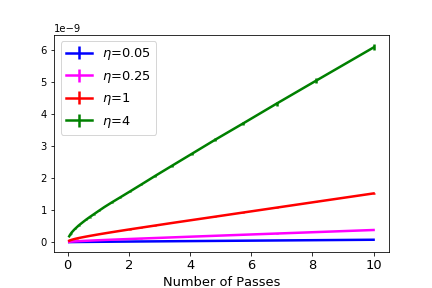}
        \caption{(b) RCD for gisette}
    \end{minipage}\hfill
    \begin{minipage}{0.24\textwidth}
        \includegraphics[width=\linewidth]{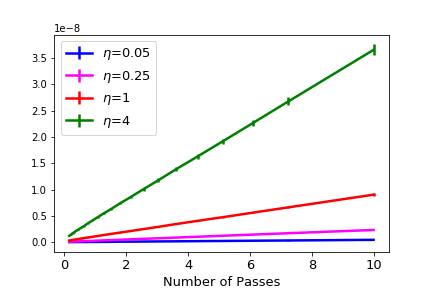}
        \caption{(c) RCD for madelon}
    \end{minipage}\hfill
    \begin{minipage}{0.24\textwidth}
        \includegraphics[width=\linewidth]{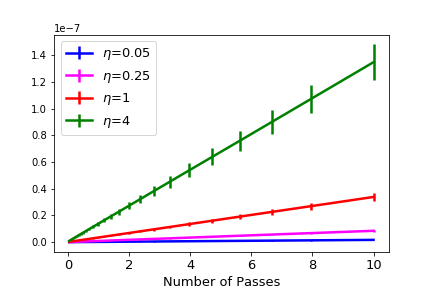}
        \caption{(d) RCD for usps}
    \end{minipage}

    \vspace{1em}

    \begin{minipage}{0.24\textwidth}
        \includegraphics[width=\linewidth]{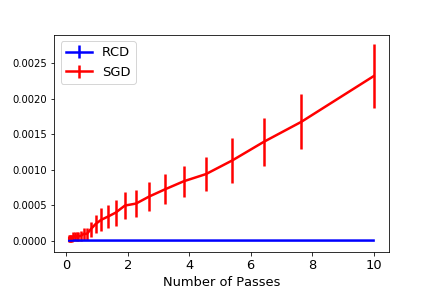}
        \caption{(e) Comparison on a3a}
    \end{minipage}\hfill
    \begin{minipage}{0.24\textwidth}
        \includegraphics[width=\linewidth]{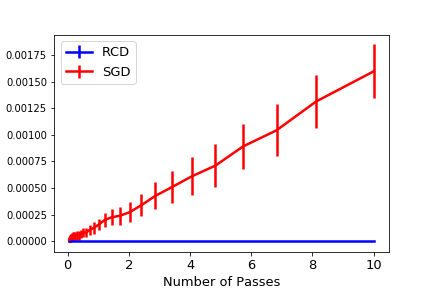}
        \caption{(f) Comparison on gisette}
    \end{minipage}\hfill
    \begin{minipage}{0.24\textwidth}
        \includegraphics[width=\linewidth]{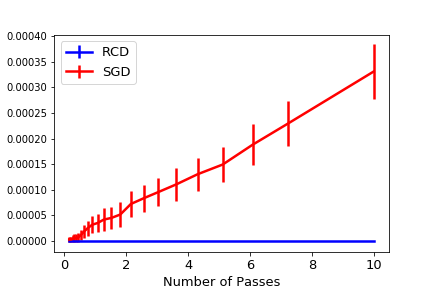}
        \caption{(g) Comparison on madelon}
    \end{minipage}\hfill
    \begin{minipage}{0.24\textwidth}
        \includegraphics[width=\linewidth]{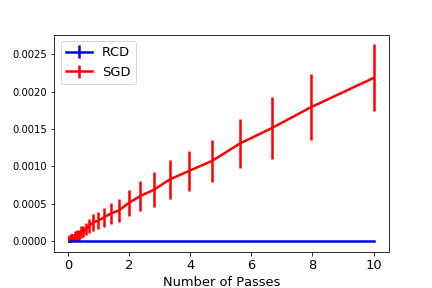}
        \caption{(h) Comparison on usps}
    \end{minipage}

    \caption{Figure E.1: Euclidean distance $\Delta_t$ as a function of the number of passes for the hinge loss.}
\end{figure}

\begin{figure}[H]
    \captionsetup{labelformat=empty}
    \captionsetup{aboveskip=17pt}
    \centering
    \begin{minipage}{0.24\textwidth}
        \includegraphics[width=\linewidth]{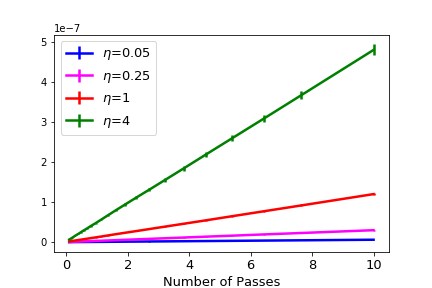}
        \caption{(a) RCD for a3a}
    \end{minipage}\hfill
    \begin{minipage}{0.24\textwidth}
        \includegraphics[width=\linewidth]{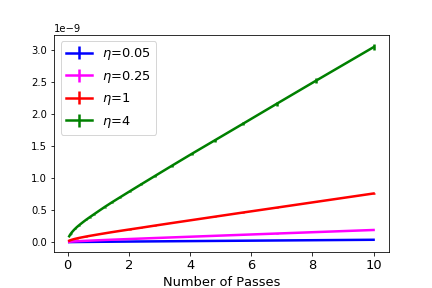}
        \caption{(b) RCD for gisette}
    \end{minipage}\hfill
    \begin{minipage}{0.24\textwidth}
        \includegraphics[width=\linewidth]{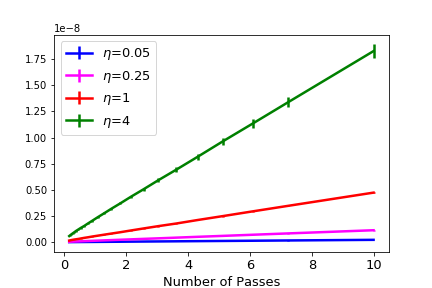}
        \caption{(c) RCD for madelon}
    \end{minipage}\hfill
    \begin{minipage}{0.24\textwidth}
        \includegraphics[width=\linewidth]{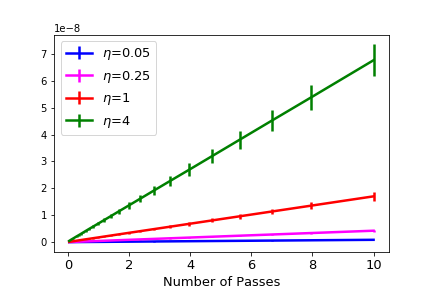}
        \caption{(d) RCD for usps}
    \end{minipage}

    \vspace{1em}

    \begin{minipage}{0.24\textwidth}
        \includegraphics[width=\linewidth]{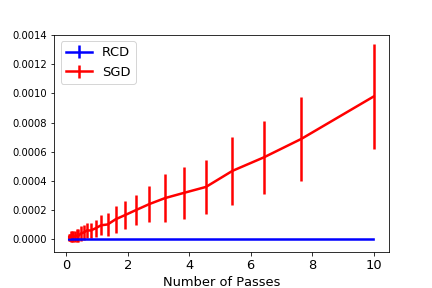}
        \caption{(e) Comparison on a3a}
    \end{minipage}\hfill
    \begin{minipage}{0.24\textwidth}
        \includegraphics[width=\linewidth]{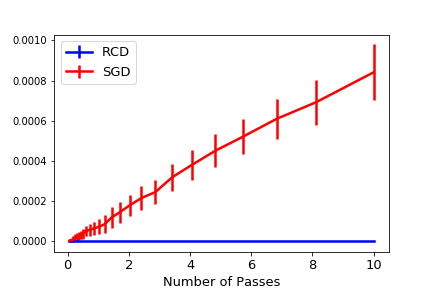}
        \caption{(f) Comparison on gisette}
    \end{minipage}\hfill
    \begin{minipage}{0.24\textwidth}
        \includegraphics[width=\linewidth]{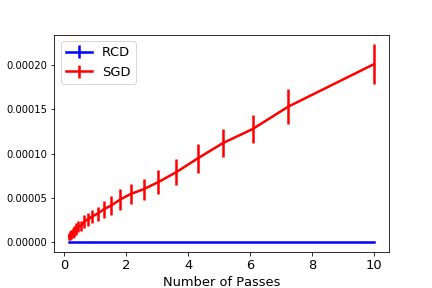}
        \caption{(g) Comparison on madelon}
    \end{minipage}\hfill
    \begin{minipage}{0.24\textwidth}
        \includegraphics[width=\linewidth]{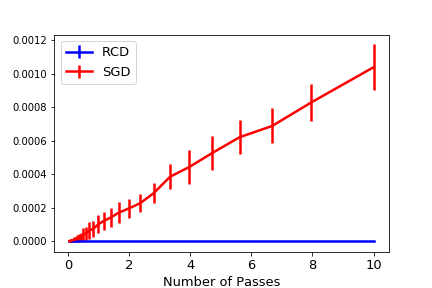}
        \caption{(h) Comparison on usps}
    \end{minipage}

    \caption{Figure E.2: Euclidean distance $\Delta_t$ as a function of the number of passes for the logistic loss.}
\end{figure}

In Figure E.1 and Figure E.2, (a), (b), (c), (d) show us that $\Delta_t$ is increasing as $t$ or $\eta$ grows.
This is consistent with the theoretical results on stability bounds in convex case.
Turning to (e), (f), (g), (h) in Figure E.1 and Figure E.2, we know that RCD is significantly more stable than SGD for pairwise learning.
While the term $(T/n)^2$ dominates the rates of stability bounds for SGD according to Theorem 3 and Theorem 6 of \citeauthor{l:2021} \shortcite{l:2021}, the on-average argument stability bound for RCD takes the order of $O(T/n^2)$.
The experiments for the comparison between RCD and SGD are also consistent with the theoretical stability bounds.
Furthermore, the Euclidean distance under the logistic loss is significantly smaller than that under the hinge loss, which is consistent with the the discussions of \citeauthor{l:2021} \shortcite{l:2021} for smooth and nonsmooth problems.

\bibliography{aaai25}

\end{document}